\documentclass[a4paper, 11pt]{article}

\usepackage{microtype}
\usepackage{graphicx}
\usepackage{subfigure}
\usepackage{booktabs} 
\usepackage{bm}
\usepackage{xfrac}
\usepackage{balance}
\usepackage{soul}

\usepackage{authblk}
\usepackage{comment}
\usepackage[british]{babel}
\usepackage[mono=false]{libertine}
\usepackage{setspace}
\usepackage[utf8]{inputenc} 
\usepackage[T1]{fontenc}    
\usepackage{url}            
\usepackage{amsfonts}       
\usepackage{amsmath,amssymb,amsthm,mathrsfs,bbm}
\usepackage{physics}
\usepackage{color}
\usepackage[usenames, dvipsnames]{xcolor}
\usepackage{balance}
\usepackage{times}
\usepackage{svg}
\usepackage{wrapfig}
\usepackage{pythonhighlight}
\usepackage{csquotes}
\usepackage{biblatex}
\usepackage{algorithmic}
\usepackage{algorithm}

\usepackage{mathtools}

\usepackage[top=2.5cm, bottom=2.5cm, left=2cm, right=2cm]{geometry}
\usepackage[pdfpagemode=UseNone,bookmarks=false, bookmarksopen=false,colorlinks=true,urlcolor=RoyalBlue,citecolor=YellowOrange,citebordercolor=RoyalBlue,linkcolor=RoyalBlue]{hyperref}

\theoremstyle{plain}
\newtheorem{theorem}{Theorem}[section]

\newtheorem{lemma}[theorem]{Lemma}

\theoremstyle{definition}

\theoremstyle{remark}

\def\dd{\text{d}}

\def\prox{{\rm prox}}

\def\out{{\rm out}}

\def\bo{{\rm bo}}
\def\erm{{\rm erm}}
\def\amp{{\rm amp}}

\def\sign{{\rm sign}}
\def\erf{{\rm erf}}
\def\erfc{{\rm erfc}}

\def\noisestr{\tau}
\def\noisevar{\noisestr^2}

\def\wstar{\vec{w}_{\star}}

\def\yteacher{y^*}

\def\prederm{\hat{f}_{\text{ERM}}}

\def\qbo{q}
\def\mbo{m}
\def\Vbo{V}
\def\hatqbo{\hat{\qbo}}
\def\hatVbo{\hat{\Vbo}}
\def\hatmbo{\hat{\mbo}}

\def\merm{\Tilde{m}}
\def\Verm{\Tilde{V}}

\def\hatVerm{\hat{\Verm}}
\def\hatmerm{\hat{\merm}}

\def\Q{Q}
\def\hatQ{\hat{\Q}}

\def\lambdaerror{\lambda_{\text{error}}}
\def\lambdaloss{\lambda_{\text{loss}}}

\def\mat#1{\text{#1}}
\renewcommand{\vec}[1]{\bm{#1}}

\title{Theoretical characterization of uncertainty \\ in high-dimensional linear classification}

\author[1]{Lucas Clart\'e}
\author[2]{Bruno Loureiro}
\author[2]{\\Florent Krzakala}
\author[1]{Lenka Zdeborov\'a}

\affil[1]{
\'Ecole Polytechnique F\'ed\'erale de Lausanne (EPFL)\\
Statistical Physics of Computation lab.\\
CH-1015 Lausanne, Switzerland
}
\affil[2]{
\'Ecole Polytechnique F\'ed\'erale de Lausanne (EPFL)\\
Information, Learning and Physics lab.\\
CH-1015 Lausanne, Switzerland
}

\date{}

\begin{document}

\maketitle

\vskip 0.3in

\begin{abstract}
Being able to reliably assess not only the \emph{accuracy} but also the \emph{uncertainty} of models' predictions is an important endeavour in modern machine learning. Even if the model generating the data and labels is known, computing the intrinsic uncertainty after learning the model from a limited number of samples amounts to sampling the corresponding posterior probability measure. Such sampling is computationally challenging in high-dimensional problems and  theoretical results on heuristic uncertainty estimators in high-dimensions are thus scarce. 
In this manuscript, we characterise uncertainty for learning from limited number of samples of high-dimensional Gaussian input data and labels generated by the probit model. 
In this setting, the Bayesian uncertainty (i.e. the posterior marginals) can be asymptotically obtained by the approximate message passing algorithm, bypassing the canonical but costly Monte Carlo sampling of the posterior.
We then provide a closed-form formula for the joint statistics between the logistic classifier, the uncertainty of the statistically optimal Bayesian classifier and the ground-truth probit uncertainty. 
The formula allows us to investigate calibration of the logistic classifier learning from limited amount of samples. 
We discuss how over-confidence can be mitigated by appropriately regularising.
\end{abstract}

\section{Introduction}
\label{sec:intro}
An important part of statistics is concerned with assessing the \emph{uncertainty} associated with a prediction based on data. Indeed, in many sensitive fields where statistical methods are widely used, trustworthiness can be as important as accuracy. The same holds true for modern applications of machine learning where liability is important, e.g. self-driving cars and facial recognition. Yet, assessing uncertainty of machine learning methods comes with many questions. Measuring uncertainty in complex architectures such as deep neural networks is a challenging problem, with a rich literature proposing different strategies, e.g. \cite{pmlr-v48-gal16, NIPS2017_9ef2ed4b, guo_calibration_2017, ritter_scalable_2018, ABDAR2021243, kristiadi_being_2020, Wilson2020, gupta_distribution_2020}.

On the side of theoretical control of the uncertainty estimators there is an extended work in the context of Gaussian processes \cite{Mackay1995, seeger2004gaussian, 10.5555/3023638.3023667} that offer Bayesian estimates of uncertainties based on a Gaussian approximation over the predictor class \cite{MacKay1992, ritter_scalable_2018, daxberger_laplace_2021}. Essentially when the posterior measure is a high-dimensional Gaussian then computation of the marginals is possible and well controlled. Beyond the setting of Gaussian posterior measures, well-established mathematical guarantees fall short in the high-dimensional regime where the number of data samples is of the same order as the number of dimensions even for the simplest models \cite{sur_modern_2018}. Sharp theoretical results on uncertainty quantification in high-dimensional models where posterior distributions are not Gaussian are consequently scarce.

In this manuscript we provide an exact characterisation of uncertainty for high-dimensional classification of data with Gaussian covariates and probit labels. There are two main sources of uncertainty in this model -- the more explicit is the noise level parametrizing the probit function, then there is the uncertainty coming from the fact that learning is done from a limited number of samples.  
Uncertainty estimation in classification problems aims to compute the probability that a given new sample has one of the labels. The most likely label is then typically chosen for prediction of the new labels, but the probability itself is of our interest here. We stress that we are interested in the uncertainty sample-wise, i.e. for every given sample, not on average. 
We address questions such as: a) How does the uncertainty of the logistic classifier compares with the actual Bayesian uncertainty when learning with a limited amount of data? b) How do these two uncertainty measures compare with the intrinsic model uncertainty due to the noise in the data-generating process? 

The key player in our analysis will be the Bayesian estimator of uncertainty corresponding to the probabilities of labels for new samples computed by averaging over the posterior distribution.  
Although in general computing the Bayesian estimator from posterior sampling can be prohibitively computationally costly in high-dimensions, we show that in the present model it can be efficiently done using a tailored approximate message passing (AMP) algorithm \cite{bayati2015lasso,rangan2011generalized}. Leveraging tools from the GAMP literature and its state evolution, we provide an asymptotic characterisation of the joint statistics between the minimiser of the logistic loss, the optimal Bayesian estimator over the data and the oracle estimator. This allows us to provide quantitative answers to questions a) \& b) above, and to study how uncertainty estimation depends on the parameters of the model, such as the regularisation, size of the training set and noise.


\paragraph{Main contributions --}
The main contributions in this manuscript are:
\begin{itemize}
    \item It is well known that the optimal Bayesian classifier for a data model with Gaussian i.i.d. covariates and probit labels is well approximated by the generalized approximate message passing (AMP) algorithm \cite{barbier_optimal_2019,javanmard2013state}. We extend these results by showing that GAMP also provides an exact sample-wise estimation of the Bayesian uncertainty when $d\to\infty$.
    \item We provide an exact asymptotic description of the joint statistics between the uncertainty of the oracle, and the one estimated by the Bayes-optimal and logistic classifiers for the aforementioned data model. This allows us to  compare these uncertainties to each other. Comparing the oracle and Bayes optimal we quantify the uncertainty coming from the limited size of the dataset. Comparing Bayesian and logistic classifiers allows us to quantify the under- or overconfidence of the later.  
    \item  
    We derive an asymptotic expression of the calibration for the Bayesian and logistic classifiers. In particular, we show that the Bayesian estimator is calibrated. For the logistic classifier, our expression allows us characterize the influence of various parameters on under- or overconfidence of the logistic classifier. 
    \item We quantify the role played by the $\ell_{2}$-regularization on uncertainty estimation. In particular, we compare  cross-validation with respect to the optimisation loss (logistic) with cross-validation with respect to the $0/1$ error.  
\end{itemize}
\paragraph{Related work -- }
\textbf{Measures of uncertainty:} Measuring uncertainty in neural networks is a challenging problem with a vast literature proposing both frequentist and Bayesian approaches \cite{ABDAR2021243}. On the frequentist side, various algorithms have been introduced to evaluate and improve the calibration of machine learning models. Some of them, such as isotonic regression \cite{zadrozny_isotonic_2002}, histogram binning \cite{zadrozny_binning_2001}, Platt scaling \cite{platt_2000} or temperature scaling \cite{guo_calibration_2017} are applied to previously trained models. Other approaches aim to calibrate models during training, using well-chosen metrics \cite{mukhoti_focal_loss_2020, liu_distance_awareness_2020}, through data augmentation \cite{thulasidasan_mixup_2020} or using the iterates of the optimiser \cite{maddox_simple_2019}. Alternatively, different authors have proposed uncertainty measures based on Bayesian estimates \cite{Mattei2019, Wilson2020}. This includes popular methods such as Bayesian dropout \cite{pmlr-v48-gal16, 10.5555/3295222.3295309}, deep ensembles \cite{NIPS2017_9ef2ed4b, malinin_ensemble_distillation_2019, liu_distance_awareness_2020} and variational inference \cite{posch2019variational}, Laplace approximation \cite{kristiadi_being_2020, daxberger_laplace_2021} and tempered posteriors \cite{Adlam2020, aitchison2021a, Kapoor2022} to cite a few. Finally, some works based on conformal inference \cite{shafer_tutorial_2008} are concerned with providing non-asymptotic and distribution-free guarantees for the uncertainty \cite{angelopoulos_learn_2021, gupta_distribution_2020}.

\noindent \textbf{Exact asymptotics: } Our theoretical analysis builds on series of developments on the study of exact asymptotics in high-dimensions. The generalised approximate message passing (GAMP) algorithm and the corresponding state evolution equations appeared in \cite{rangan2011generalized,javanmard2013state}. Exact asymptotics for Bayesian estimation in generalised linear models was rigorously established in \cite{barbier_optimal_2019}. On the empirical risk minimisation side, exact asymptotics based on different techniques, such as Convex Gaussian Min-Max Theorem (GMMT) \cite{candes_phase_2018, 9053524, pmlr-v108-taheri20a, aubin_generalization_2020, pmlr-v119-mignacco20a, montanari2020generalization, loureiro_learning_2021, liang_precise_2020}, Random Matrix Theory \cite{8683376}, GAMP \cite{pmlr-v125-gerbelot20a, loureiro2021learning} and first order expansions \cite{NEURIPS20190609154f} have been used to study high-dimensional logistic regression and max-margin estimation.

\noindent \textbf{Uncertainty \& exact asymptotics:} An early discussion on the variance of high-dimensional Bayesian linear regression has been appeared in \cite{NIPS1994_e6cb2a3c, Marion_1995, Bruce_1994}. Calibration has been studied in the context of high-dimensional unregularised logistic regression in \cite{bai_dont_2021}, where it was shown that the logistic classifier is systematically overconfident in the regime where number of samples is proportional to the dimension. An equivalent result for regression was discussed in \cite{bai2021understanding}, where it was shown that quantile regression suffers from an under-coverage bias in high-dimensions. While \cite{bai_dont_2021} is the closest to the present paper, we differ from their setting in three major ways. First, they consider the behavior of unpenalized logistic regression, while we study the effect of $\ell_2$ regularization on uncertainty. Second, we compute the full joint distribution of the prediction for the oracle, the empirical risk minimizer and the Bayes optimal estimator, while \cite{bai_dont_2021} focus the discussion on the calibration of the empirical risk minimizer with respect to the oracle only. Lastly (and less importantly), \cite{bai_dont_2021} considers logit data, while we consider a probit data model.

\paragraph{Notation --} Vectors are denoted in bold. $\mathcal{N}(\vec{x}|\vec{\mu}, \Sigma)$ denotes the Gaussian density. $\odot$ denotes the (component-wise) Hadamard product. $\mathbf{1}(A)$ denotes the indicator on the set $A$.\looseness=-1
\section{Setting}
\label{sec:setting}
\noindent \textbf{The data model --}
Consider a binary classification problem where $n$ samples $(\vec{x}^{\mu}, y^{\mu})\in\mathbb{R}^{d}\times\{-1,1\}$, $\mu=1,\cdots, n$ are independently drawn from the following probit model:
\begin{align}
    f_{\star}(\vec{x}) \coloneqq\mathbb{P}(y^{\mu}=1|\vec{x}^\mu) = \sigma_{\star}\left(\frac{\wstar^{\top}\vec{x}^{\mu}}{\noisestr}\right), \\
\vec{x}^{\mu} \sim\mathcal{N}(\vec{0},\sfrac{1}{d}\mat{I}_{d}), \quad 
\wstar\sim\mathcal{N}(\vec{0},\mat{I}_{d})\label{eq:def:data}
\end{align}
where $\sigma_{\star}(x)=\sfrac{1}{2}~\erfc(-\sfrac{x}{\sqrt{2}})$ and $\noisestr \geq 0$ parametrises the noise level. 
Note that the probit model is equivalent to generating the labels via $y^\mu = f_{0}( \wstar^{\top}\vec{x}^{\mu} + \noisestr \xi^\mu )$ with $\xi^\mu \sim \mathcal{N}(0,1)$ and $f_0(x) := \sign(x)$. In the following we will be referring to the function $f_{\star}(\vec{x})$ or to its parameters $\wstar$ as the \textit{teacher}, having in mind the teacher-student setting from neural networks. We will refer to $f_{\star}(\vec{x})$ as the \textit{oracle uncertainty} as it takes into account only the noise in the label-generating process, but it does not take into account uncertainty coming from the limited size of the training dataset. 

Note that our discussion could be straighforwardly generalized to a generic prior distribution $\wstar\sim P_{\wstar}$. However, our goal in this work is to provide a fair comparison between Bayesian estimation and empirical risk minimization. Indeed, ERM does not assume any information on the components of $\wstar$, and a fair comparison is to consider the maximum entropy Gaussian prior.

Given the training data $\mathcal{D}=\{(\vec{x}^{\mu},y^{\mu})\}_{\mu=1}^{n}$ and a test sample $\vec{x}\sim\mathcal{N}(\vec{0},\sfrac{1}{d}\mat{I}_{d})$, the goal is to find a (probabilistic) classifier $\vec{x} \mapsto  \hat{y}(\vec{x})$ minimizing the $0/1$ test error 
\begin{equation}
    \varepsilon_g = \mathbb{E}_{(\vec{x}, y)} \mathbb{P}\left(\hat{y}(\vec{x})\neq y \right).
\end{equation}

\noindent \textbf{Considered classifiers --}
We will focus on comparing two probabilistic classifiers $\hat{f}(\vec{x}) = \mathbb{P}(y=1|\vec{x})$. The first is the widely used logistic classifier: $\hat{f}_{\erm}(\vec{x}) = \sigma(\hat{\vec{w}}_{\erm}^{\top}\vec{x})$ where $\sigma(x)=(1+e^{-x})^{-1}$ is the logistic function and the weights $\hat{\vec{w}}\in\mathbb{R}^{d}$ are obtained by minimising the following (regularised) empirical risk:
\begin{align}
\hat{\mathcal{R}}_{n}(\vec{w}) = \frac{1}{n}\sum\limits_{\mu=1}^{n}\log\left(1+e^{-y^{\mu}\vec{w}^{\top}\vec{x}^{\mu}}\right)+\frac{\lambda}{2}||\vec{w}||^{2}_{2}.\label{eq:def:erm}
\end{align}
Using $\hat{f}_{\erm}(\vec{x})$ as a measure of uncertainty is not considered very principled. Never-the-less it is arguably the most commonly used measure to give a rough idea of how confident is the neural network prediction for a given sample.  

The second estimator we investigate is the statistically optimal Bayesian estimator for the problem, which is given by:\looseness=-1
\begin{align}
   \hat{f}_{\bo}(\vec{x}) &= \mathbb{P}_{\text{BO}}(y=1|\vec{x}) =\!\! \int_{\mathbb{R}^{d}}\!\!\dd\vec{w} ~ P(y=1| \vec{x}^{\top}\vec{w})P(\vec{w}|\mathcal{D})\notag\\
        &= \int_{\mathbb{R}^{d}}\dd\vec{w} ~\sigma_{\star}\left(\frac{\vec{w}^{\top}\vec{x}}{\noisestr}\right)P(\vec{w}|\mathcal{D})\, ,
           \label{def:fbo}
\end{align}
\noindent where the posterior distribution $P(\vec{w}|\mathcal{D})$ given the training data $\mathcal{D}=\{(\vec{x}^{\mu},y^{\mu})\}_{\mu=1}^{n}$ is explicitly given by:
\begin{align}
    P(\vec{w}|\mathcal{D}) = \frac{1}{\mathcal{Z}(\noisestr)}\prod\limits_{\mu=1}^{n}\sigma_{\star}\left(y^{\mu}\frac{\vec{w}^{\top}\vec{x}^{\mu}}{\noisestr}\right)\mathcal{N}(\vec{w}|\vec{0},\mat{I}_{d}),\label{eq:def:posterior}
\end{align}
\noindent for a normalisation constant $\mathcal{Z}(\noisestr)\in\mathbb{R}$. 
The Bayes-optimal (BO) estimator $\hat{f}_{\bo}(\vec{x})$ provides the perfect measure of uncertainty that takes into account both the noise in the data generation and the finite number of samples in the training set. The traditional drawback of course is that it assumes the knowledge of the value $\noisestr$ and other details of the data-generating model.


\noindent \textbf{Uncertainty and calibration --}
The main purpose of this manuscript is to characterise how the intrinsic uncertainty of the probit model compares to both the Bayesian and logistic confidences/uncertainties in the high-dimensional setting where the number of samples $n$ is comparable to the dimension $d$. In this case, the limited number of samples is a sources of uncertainty comparable in magnitude to the noise level $\tau$. To define what is uncertainty in our context, note that the \textit{confidence functions} $\hat f(\vec{x}) = {\mathbb P}(y=1|\vec{x})$ defined above give the probability that the label is $y=1$ (with the label prediction commonly given by thresholding this function). In mathematical terms, we aim at characterising the correlation between the oracle, Bayesian and logistic confidences, as parametrised by the joint probability density:
\begin{align}
    \rho(a,b,c) \!=\! \mathbb{P}_{\mathcal{D}, \vec{x}}\big({f}_{\star}(\vec{x})\!=\!a, \hat{f}_{\bo}(\vec{x}) \!=\! b,  \hat{f}_{\erm}(\vec{x})\!=\!c \big)\, . \label{eq:def:jointdensity}
\end{align}
Similarly, we will note $\rho_{\star, \erm}(a, c) = \mathbb{P}(f_{\star} = a, \hat{f}_{\erm} = c)$, $\rho_{\bo, \erm}(b, c) = \mathbb{P}(\hat{f}_{\bo} = b, \hat{f}_{\erm} = c)$ and $\rho_{\star, \bo}(a, b) = \mathbb{P}( f_{\star} = a, \hat{f}_{\bo} = b)$. These densities correspond to $\rho$ summed over $\hat{f}_{\bo}$, $f_{\star}$ and $\hat{f}_{\erm}$ respectively. 
Here the sample $\vec{x}$ is understood as any sample from the test set, on which the confidence/uncertainty is evaluated. It is important that Eq.~\eqref{eq:def:jointdensity} is defined for the same sample $\vec{x}$ in all the 3 arguments. 
Not that $\rho_{*, erm}$ allows to compare the ERM uncertainty with the oracle uncertainty (the best we could do if we had infinite data), while $\rho_{bo, erm}$ quantifies the ERM uncertainty with respect to the best statistical estimate under a finite amount of data. 

In the next Section, we provide a characterisation of this joint density in the high-dimensional limit where $n,d\to\infty$ with fixed sample complexity $\alpha = \sfrac{n}{d}$, as a function of the noise level $\noisestr$ and regularization $\lambda$. To obtain this result we leverage recent works on approximate message passing algorithms and their state evolution. 

Some of our results will be conveniently formulated in terms of so-called calibration of a probabilistic classifier $\hat{f}:\mathbb{R}^{d}\to [0,1]$ defined as:\looseness=-1
\begin{equation}
    \Delta_p(\hat{f}) := p - \mathbb{E}_{\vec{x}, \yteacher}( f_{\star}(\vec{x}) | \hat{f}(\vec{x}) = p) \label{eq:def_calibration} 
\end{equation}
where $\hat{f}$ can be the logistic classifier or the Bayes-optimal one. 
Intuitively, the calibration quantifies how well the predictor assigns probabilities to events. If $\Delta_{p} = 0$ the predictor is said to be \textit{calibrated at level p}. Instead, if for $p > \sfrac{1}{2}, \Delta_p > 0$ (respectively $\Delta_p < 0$), then the predictor is said to be \textit{overconfident} (respectively \textit{underconfident}) Note, however, that the calibration is an average notion, while the above joint probability distribution \eqref{eq:def:jointdensity} captures more detailed information about the point-wise confidence and its reliability. In this work, we will also consider the calibration of ERM with respect to Bayes
\begin{equation}
     \tilde{\Delta}_p := p - \mathbb{E}_{\vec{x}, \yteacher}(\hat{f}_{\bo}(\vec{x}) | \hat{f}_{\erm}(\vec{x}) = p)
     \label{eq:calibration_erm_bayes}
\end{equation}
Finally, while our discussion focuses in the calibration for concreteness, note that many other uncertainty metrics could be studied from the joint density eq.~\eqref{eq:def:jointdensity}.

\section{Technical theorems}
\label{sec:mainres}
Our first technical result is the existence of an efficient algorithm (Algorithm \ref{alg:gamp}), called Generalized Approximate Message Passing (GAMP))  \cite{rangan2011generalized,javanmard2013state} that is able to accurately estimate $\hat{f}_{\bo}(\vec{x})$ in high-dimensions. The asymptotic accuracy of GAMP for the Bayes-optimal average (over the samples) test error is know from \cite{barbier_optimal_2019}. In order to formulate our results we also need to prove that the probabilities estimated by GAMP are also accurate \textit{sample-wise}, this relatively straightforward extension of the results of \cite{barbier_optimal_2019} is covered by the following lemma: 
\begin{lemma}[Sample-wise GAMP-Optimality]
\label{thm:gamp}
  For a sequence of problems given by eq.~\eqref{eq:def:data}, and given the estimator $\hat{\vec{w}}_{\amp}$ from Algorithm 1, the predictor
 \begin{equation}
 \hat{f}_{\amp} (\vec{x}) = {\mathbb P}(y=1|\vec{x}) = \sigma_{\star}\left(\frac{\hat{\vec{w}}_{\amp}^{\top} \vec{x}}{\sqrt{\noisevar + \hat{\vec{c}}_{\amp}^{\top} (\vec{x} \odot \vec{x})}}\right)
 \label{def:amp}
 \end{equation}
 is such that, with high probability over a new sample $\bf x$ the classifier above is asymptotically equal to the Bayesian estimator $\hat{f}_{\bo}(\vec{x}) = {\mathbb P}(y=1|\vec{x})=\hat{f}_{\amp} (\vec{x})$ in eq.~\eqref{def:fbo}. More precisely:
\begin{equation}
   \forall \varepsilon > 0, \lim_{d \to \infty} \mathbb{P}_{\vec{x},\mathcal D} \left( |\hat f_{\rm amp}(\vec{x})-\hat f_{\bo}(\vec{x})|^2 \leqslant \varepsilon \right) \to 1
\end{equation}
In particular, the predictor $\hat{f}_{\amp}$ asymptotically achieves the best possible test performance (the one achieved by the Bayes-optimal estimator)
\end{lemma}
\begin{algorithm}[bt]
   \label{alg:gamp}
   \caption{GAMP}
\begin{algorithmic}
   \STATE {\bfseries Input:} Data $\mat{X}\in\mathbb{R}^{n\times d}$, $\vec{y}\in\{-1,1\}^{n}$ 
   
   \STATE Define $\mat{X}^2 = \mat{X}\odot \mat{X} \in\mathbb{R}^{n\times d}$ and Initialize $\hat{\vec{w}}^{t=0} = \mathcal{N}(\vec{0}, \sigma_{w}^2\mat{I}_{d})$, $\hat{\vec{c}}^{t=0} = \vec{1}_{d}$, $\vec{g}^{t=0} = \vec{0}_{n}$.
   \FOR{$t\leq t_{\text{max}}$}
   
   \STATE $\vec{V}^{t} = \mat{X}^{2} \hat{\vec{c}}^{t}$ ; $\vec{\omega}^{t} = \mat{X} \hat{\vec{w}}^{t} - \vec{V}^{t}\odot \vec{g}^{t-1}$ ; \qquad\textit{/* Update channel mean and variance}
    
   \STATE $\vec{g}^{t} = f_{\text{out}}(\vec{y}, \vec{w}^{t}, \vec{V}^{t})$ ; $\partial\vec{g}^{t} = \partial_{\omega}f_{\text{out}}(\vec{y}, \vec{w}^{t}, \vec{V}^{t})$ ; \qquad\textit{/* Update channel}
   
   \STATE $\vec{A}^{t} = -{\mat{X}^{2}}^{\top} \partial \vec{g}^{t}$ ; $\vec{b}^{t} = \mat{X}^{\top} \vec{g}^{t} + \vec{A}^{t}\odot \hat{\vec{w}}^{t}$ ; \qquad\textit{/* Update prior mean and variance } 
   \STATE \textit{/* Update marginals */}
   \STATE $\hat{\vec{w}}^{t+1} = f_w(\vec{b}^{t}, \vec{A}^{t}) \coloneqq (\mat{I}_{d} + \vec{A}^{t})^{-1}\vec{b}^{t}$ ;\qquad $\hat{\vec{c}}^{t+1} = \partial_{b}f_w(\vec{b}^{t}, \vec{A}^{t}) \coloneqq (\mat{I}_{d}+\vec{A}^{t})^{-1}$
   
   \ENDFOR
   \STATE {\bfseries Return:} Estimators $\hat{\vec{w}}_{\amp}, \hat{\vec{c}}_{\amp}\in\mathbb{R}^{d}$
\end{algorithmic}
\end{algorithm}

The proof of Lemma \ref{thm:gamp} is provided in Appendix~\ref{sec:app:proofs}. As mentioned above, the lemma does not require the prior on $\wstar$ to be Gaussian. Changing the prior of $\wstar$ amount to changing the denoising functions $(f_{w},\partial_{b}f_{w})$ in Algorithm~\ref{alg:gamp}. Similarly, the probit likelihood defined in equation \eqref{eq:def:data} is not required for our analysis. In fact, the equations hold for any probabilistic generalized linear model, and in particular for the logit data model studied in \cite{bai_dont_2021}, reproduced in Appendix \ref{app:logit}. This choice of likelihood function only changes the denoising \textit{channel} functions $(f_{out}, \partial_{\omega} f_{out})$. The motivation behind the use of the GAMP Algorithm is twofold. First, it allows us to characterize the posterior mean needed to express the probability $\hat{f}_{\amp} (\vec{x}) $ for a given new sample $\vec{x}$ in polynomial time in $d$. Indeed, each iteration of the loop in Algorithm \ref{alg:gamp} is $O(d^2)$. Second, the asymptotic performance of GAMP is conveniently tracked by low-dimensional \textit{state evolution} equations which can be easily solved in a computer.

Our second technical result is a formula for the joint distribution of the teacher label, its Bayes estimate, and the estimate from empirical risk minimisation defined in eq.~(\ref{eq:def:jointdensity}), described in the following theorem:
\begin{theorem}
\label{thm:jointstats}
Consider training data $\mathcal{D}=\{(\vec{x}^{\mu},y^{\mu})\}_{\mu=1}^{n}$ sampled from the model defined in eq.~\eqref{eq:def:data}. Let $\hat{\vec{w}}_{\erm}\in\mathbb{R}^{d}$ be the solution of the empirical risk minimisation \eqref{eq:def:erm} and $\hat{\vec{w}}_{\amp}$ denote the estimator returned by running algorithm \ref{alg:gamp} on the data $\mathcal{D}$. Then in the high-dimensional limit where $n,d\to\infty$ with $\alpha = \sfrac{n}{d}$ fixed, the asymptotic joint density \eqref{eq:def:jointdensity} is given by:
\begin{align}
    \rho(a, b, c) = {\noisestr}' \noisestr \frac{\mathcal{N}\left( \begin{bmatrix} \noisestr \cdot \sigma_{\star}^{-1}(a) \\ {\noisestr}' \cdot \sigma_{\star}^{-1}(b) \\ \sigma^{-1}(c)\end{bmatrix} \Big| \mathbf{0}_3, \Sigma \right)}{|\sigma_{\star}'(\sigma_{\star}^{-1}(a))||\sigma_{\star}'(\sigma_{\star}^{-1}(b))| |\sigma'(\sigma^{-1}(c))|}  \label{eq:res:jointdensity}
\end{align}
\noindent where we noted 
\begin{equation}
    {{\noisestr}'}^2 = \noisevar + 1 - q_{\bo}, \qquad \Sigma = \begin{bmatrix} 1 & q_{\bo} & m \\ q_{\bo} & q_{\bo} & m \\ m & m & q_{\rm erm}\end{bmatrix}
    \label{eq:def_sigma}
\end{equation}
and the so-called overlaps:
\begin{align}
    & q_{\bo} = \frac{1}{d}\hat{\vec{w}}_{\amp}^{\top}\wstar= \frac{1}{d}||\hat{\vec{w}}_{\amp}||^2_{2} \label{def-ov1} \\ & m = \frac{1}{d}\hat{\vec{w}}_{\erm}^{\top}\wstar, \quad q_{\erm} = \frac{1}{d}||\hat{\vec{w}}_{\erm}||^2_{2}
    \label{def-ov2}
\end{align}
solve the following set of self-consistent equations:
\begin{align} 
        \frac{1}{q_{\bo}}= 1 + \alpha ~\mathbb{E}_{(z, \eta), \xi}\left[f_{{\rm out}}(f_{0}(z + \noisestr \xi), \eta, 1 - q_{\bo})^2\right],
        \label{def:eq-bo}
\end{align}
and 
\begin{align}
\label{def:eq-erm-overlaps}
		V = \frac{1}{\lambda+\hat{V}}, &&
		q_{\erm} = \frac{\hat{m}^2+\hat{q}}{(\lambda+\hat{V})^2},&&
		m = \frac{\hat{m}}{\lambda+\hat{V}}.
\end{align}
\begin{align}
\label{def:eq-erm-hats}
\begin{cases}
		\hat{V} &=-\alpha \mathbb{E}_{(z, \omega), \xi}\left[\partial_{\omega}f_{\erm}(f_{0}(z + \noisestr \xi), \omega, V)\right]\\
		\hat{q} &=\alpha\mathbb{E}_{(z, \omega), \xi}\left[f_{\erm}(f_{0}(z + \noisestr \xi), \omega, V)^2\right]\\
		\hat{m} &=\alpha\mathbb{E}_{(z, \omega), \xi}\left[f_{\erm}(f_{0}(z + \noisestr \xi), \omega, V)\right]
	\end{cases}
\end{align}
\noindent where $(z,\eta, \omega) \sim\mathcal{N}\left(0_{3},\Sigma\right)$, $\xi \sim \mathcal{N}(0, 1)$ and the thresholding functions are defined as 
\begin{align}
    f_{{\rm out}}(y,\omega,V) &= \frac{2y~\mathcal{N}(\omega y|0, V+\noisevar)}{\erfc\left(-\frac{y\omega}{\sqrt{2(\noisevar+V)}}\right)}\notag\\
    f_{\erm}(y, w, V) &= V^{-1} \left( \prox_{V l(y, .)}(w) - w \right)
    \label{eq:definition_f_out}
\end{align}
with $\prox_{\tau f}(x) = {\rm argmin}_{z} \left( \sfrac{1}{2\tau}\|z-x \|_2^2 + f(z)\right)$ being the proximal operator. 
\end{theorem}

In Appendix \ref{sec:app:cavity} we show how this result can be deduced directly from the heuristic cavity method, and the analysis of the GAMP state evolution to compute the overlaps of ERM and BO estimators. To compute the correlation between the ERM and BO estimators, we use the  Nishimori identity \cite{iba1999nishimori,zdeborova2016statistical}. More details, as well as the formal proof, are given in Appendix~\ref{sec:app:proofs}.

Our third theorem is an asymptotic expression for the calibration error. 
\begin{theorem}
\label{thm:calibration}
The analytical expression of the joint density $\rho$ yields the following expression for the calibration $\Delta_p$:
    \begin{equation}
        \Delta_p(\hat{f}_{\erm}) = p - \sigma_{\star} \left( \frac{\sfrac{m}{q_{\erm}} \times \sigma^{-1}(p)}{\sqrt{1 - \sfrac{m^2}{q_{\erm}} + \noisevar}} \right)  \, . 
        \label{eq:thm_cal_erm}
    \end{equation}
Moreover, the Bayesian classifier is always well calibrated with respect to the teacher, meaning:
    \begin{equation}
        \forall p\in [0,1],\quad \Delta_p(\hat{f}_{\bo}) = 0 \, . 
        \label{eq:thm_cal_bo}
    \end{equation}
Additionally, the calibration of ERM with respect to the Bayesian classifier and the oracle are equal: 
    \begin{equation}
        \forall p\in[0,1],\quad \Delta_p(\hat{f}_{\erm}) = \tilde{\Delta}_p\, . 
        \label{eq:thm_cal_erm_bo}
    \end{equation}
\end{theorem}


The proof of Theorem \ref{thm:calibration} is provided in appendix \ref{sec:proof_thm_calibration}. Equation \eqref{eq:thm_cal_erm} shows the different factors that influence $\Delta_p$: the aleatoric uncertainty represented by the noise $\noisestr^2$, the finiteness of data that appears through $\sfrac{m}{q_{\erm}}$ and $\sfrac{m^2}{q_{\erm}}$, and the mismatch in the model with the activations $\sigma_{\star}, \sigma$. Moreover, Equation \eqref{eq:thm_cal_erm_bo} provides a recipe to compute the calibration $\Delta_p$ in the high-dimensional limit from the knowledge of the data model \eqref{eq:def:data} only, but without knowing the specific realisation of the weights $\wstar$. This is because the quantities $q_{\bo}$, $q_{\erm}$ and $m$ self-average as $n, d \rightarrow \infty$, we then obtain the calibration $\Delta_p$ without knowing the realisation of $\wstar$. 


\section{Results for uncertainty estimation}
\label{sec:experiments}

\subsection{Bayes versus oracle uncertainty}
We now discuss the consequences of the theorems from Section \ref{sec:mainres}.
Figure \ref{fig:teacher_bo_density} left panel depicts the theoretical prediction of the joint density $\rho_{\bo, \star}$, between the Bayes posterior confidence/uncertainty $\hat f_\bo$ (x-axes) and the oracle confidence/uncertainty $f_\star$ (y-axes). 
The theoretically derived density (Figure~\ref{fig:teacher_bo_density} left panel) is compared to its numerical estimation in Figure~\ref{fig:teacher_bo_density} right panel, computed numerically using the GAMP algorithm. To estimate the numerical density in the right panel, we proceed as follow: after fixing the dimension $d$ and the number of training samples $n = \alpha d$, GAMP is ran on the training set. Once GAMP estimators have been obtained, $n_{\rm test}$ test samples are drawn and for each of them we compute the confidence of the oracle/teacher $f_{\star}(\vec{x})$ from eq.~\eqref{eq:def:data}, and the Bayesian confidence $\hat{f}_{\bo}(\vec{x}) = \hat{f}_{\amp}(\vec{x})$ from Theorem \ref{thm:gamp}. Finally we plot the histogram of the thus obtained joint density $\rho_{\bo, \star}$ over the test samples. As the figure shows, there is a perfect agreement between theory and finite instance simulations.

We see that the density is positive in the whole support, it is peaked around $(0,0)$ and $(1,1)$, but has a notable weight around the diagonal as well. The relatively large spread of the joint density is a consequence of the fact that on top of the intrinsic uncertainty of the teacher, the learning is only done with $n=\alpha d$ samples which brings an additional source of uncertainty captured in the Bayes estimator. Fig.~\ref{fig:teacher_bo_density} thus quantifies this additional uncertainty due to finite $\alpha$. We are not aware of something like this being done analytically in previous literature. 

The blue curve is the mean of $f_{\star}$ conditioned on the values of $\hat{f}_{\bo}$. The difference between this and the diagonal is the calibration $\Delta_p$ defined in Equation \eqref{eq:def_calibration}. We see that the figure illustrates  $\Delta_p(\hat{f}_{\bo}) = 0$, i.e. the Bayesian prediction is well calibrated, as predicted by Theorem~\ref{thm:calibration}. 

\begin{figure}[t!]  
    \centering
    \subfigure[]{\includegraphics[width=0.448\columnwidth]{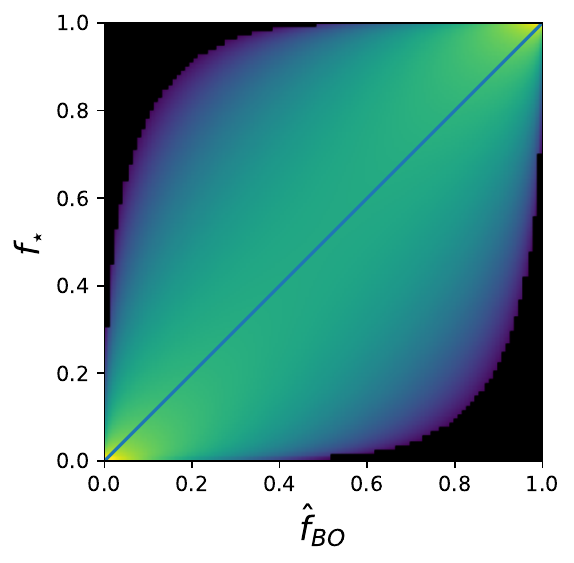}}
    \subfigure[]{\includegraphics[width=0.532\columnwidth]{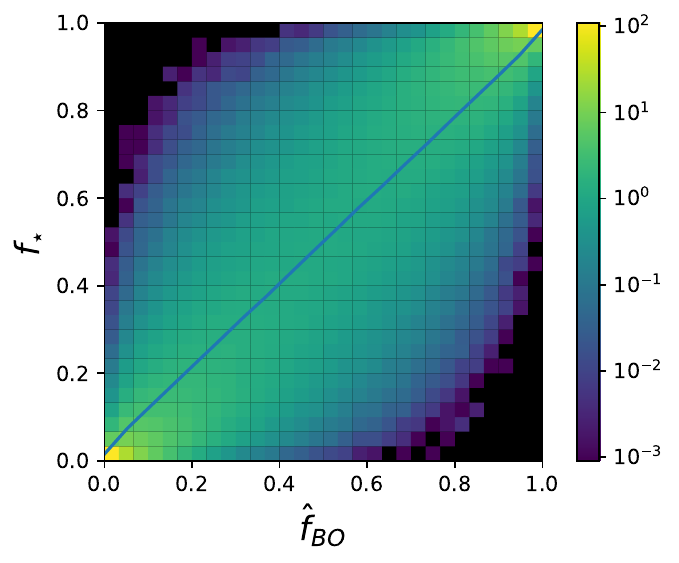}}
    \caption{Theoretical prediction (left panel) and numerical estimation (right panel) of the joint density $\rho_{\bo, \star}$ at $\alpha = 10$ and noise level $\noisestr = 0.5$. Numerical plot was done by running GAMP at dimension $d = 1000$, computing $(f_{\star}, \hat{f}_{\bo})$ on $n_{\rm test}=10^7$ test samples. The blue curve is the mean of $f_{\star}$ given $\hat{f}_{\bo}$. For these parameters, the test error of Bayes is $\varepsilon^{\bo}_g = 0.173$, the oracle test error $\varepsilon^{\star} = 0.148$.}
    \label{fig:teacher_bo_density}
\end{figure}

\begin{figure}[t!]
     \centering
     \includegraphics[width = \columnwidth]{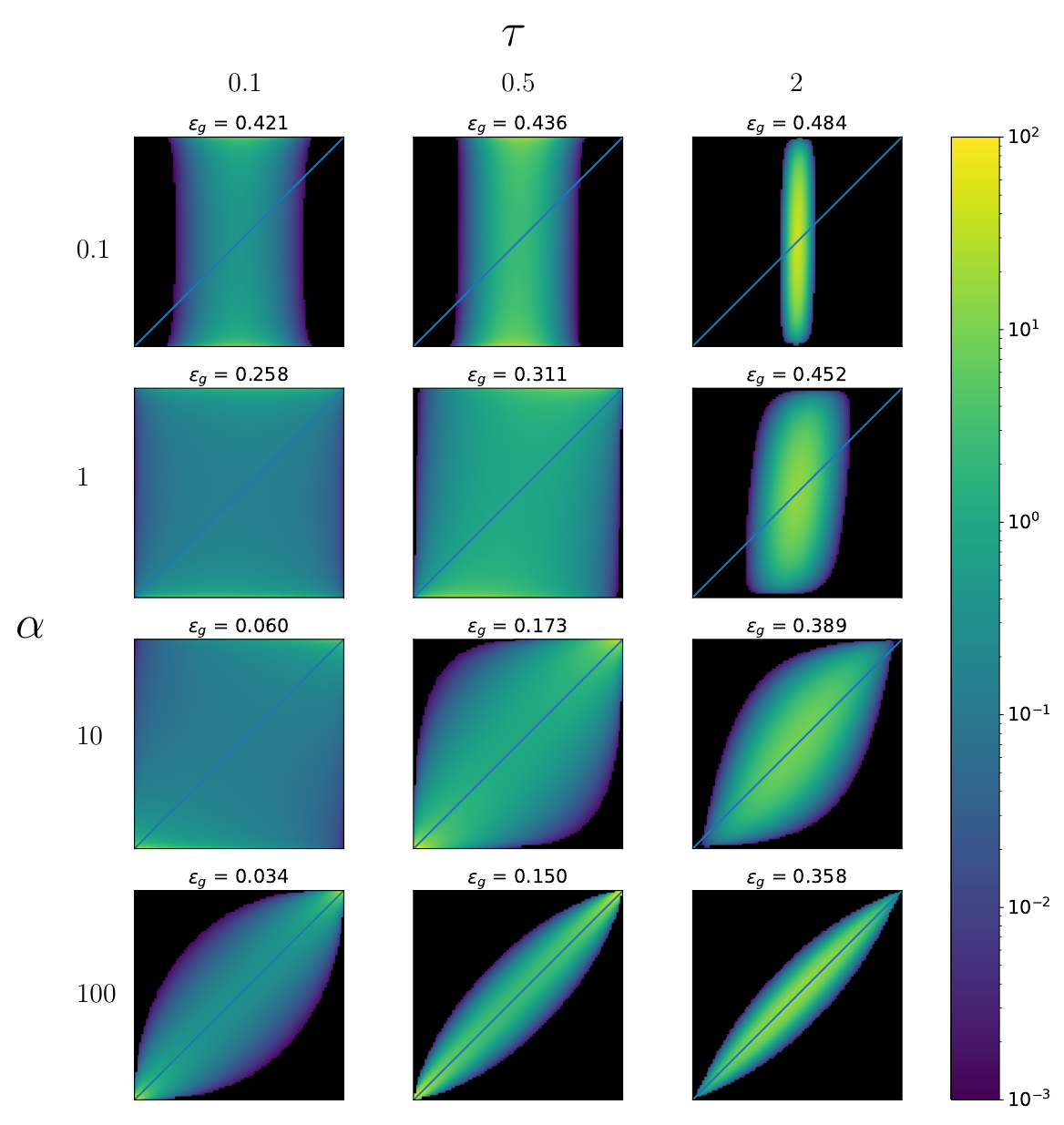}
     \caption{Density between Bayes confidence $\hat{f}_{\bo}$ (x-axis) and the oracle confidence $f_{\star}$ (y-axis) for multiple values of $\alpha, \noisestr$: the rows correspond  respectively to $\alpha = 0.1, 1, 10, 100$ from top to bottom, and the columns correspond respectively to $\noisestr = 0.1, 0.5, 2$. Generalisation errors of the Bayes estimator are in written on top of the corresponding plot. The best possible generalisation errors, achieved if the teacher weights are known, for $\noisestr = 0.1, 0.5, 2$ are respectively $\varepsilon^{\star}_g = 0.032, 0.148, 0.352$.}
     \label{fig:multiple_densities_main}
\end{figure}

Figure~\ref{fig:multiple_densities_main} then depicts the same densities as Figure~\ref{fig:teacher_bo_density} for several different values of the sample complexity $\alpha$ and noise $\noisestr$. The corresponding test error is given for information. We see, for instance, that at small $\alpha$ the BO confidence is low, close to $0.5$, because not much can be learned from very few samples. The oracle confidence does not depend on $\alpha$, and is low for growing $\noisestr$. At large $\alpha$, on the other hand, the BO confidence is getting well correlated with the oracle one. At larger $\alpha$ and small noise the BO test error is getting smaller and the corresponding confidence close to $1$ or $0$ (depending on the label). The trends seen in this figure are expected, but again here we quantify them in an analytic form of eq.~\eqref{eq:res:jointdensity} which as far as we know has not been done previously.

\subsection{Logistic regression uncertainty and calibration}
Having explicit access to the Bayesian confidence/uncertainty in a high-dimensional setting is a unique occasion to quantify the quality of the logistic classifier, which has its own natural measure of confidence induced by the logit. How accurate is this measure? We start with the logistic classifier at zero regularization and then move to the regularised case in the next section.   

Figure \ref{fig:erm_joint_density_lambda_0} compares the joint density of $(\hat{f}_{\erm}, f_{\star})$ (left panel), and $(\hat{f}_{\erm}, \hat{f}_{\bo})$ (right panel) with the same noise and number of samples as used in figure \ref{fig:teacher_bo_density}. The blue curves are the means of $f_{\star}$ (respectively $\hat{f}_{\bo}$) conditioned on $\hat{f}_{\erm}$, their shape is demonstrating that the (non-regularized) logistic classifier is on average overconfident, as is well known in practice. 

The equality between these two blue curves illustrates Theorem \ref{thm:calibration}, Equation \eqref{eq:thm_cal_erm_bo}: $\Delta_p(\hat{f}_{\erm}) = \tilde{\Delta}_p$.  Note, however, that while the calibrations of the ERM with respect to the oracle or the BO are equal, the conditional variances of $f_{\star}$ and $\hat{f}_{\bo}$ are very different. This shows how the calibration is only a very partial fix of the confidence estimation for ERM: when $\hat{f}_{\erm} = p$, both Bayes and the oracle's predictions will be $p - \Delta_p$ \textit{on average}, but for the considered parameters the predictions of the oracle are much more spread around this value than those of Bayes estimator. This means that the ERM still captures rather well some part of the uncertainty coming from the limited number of samples.  
Figure \ref{fig:multiple_densities_bo_erm} in the Appendix \ref{sec:additional_figures} complements Figure \ref{fig:erm_joint_density_lambda_0} by showing other values of $\alpha$ and $\tau$. 

\begin{figure}[t!]
    \centering
    \subfigure[]{\includegraphics[width=0.448\columnwidth]{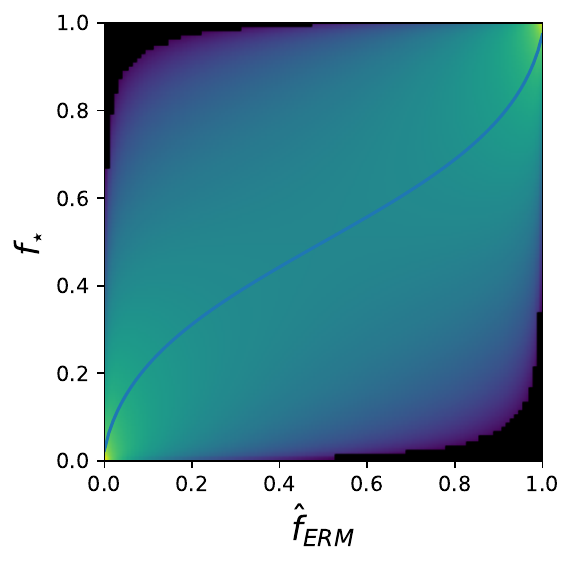}}
    \subfigure[]{\includegraphics[width=0.532\columnwidth]{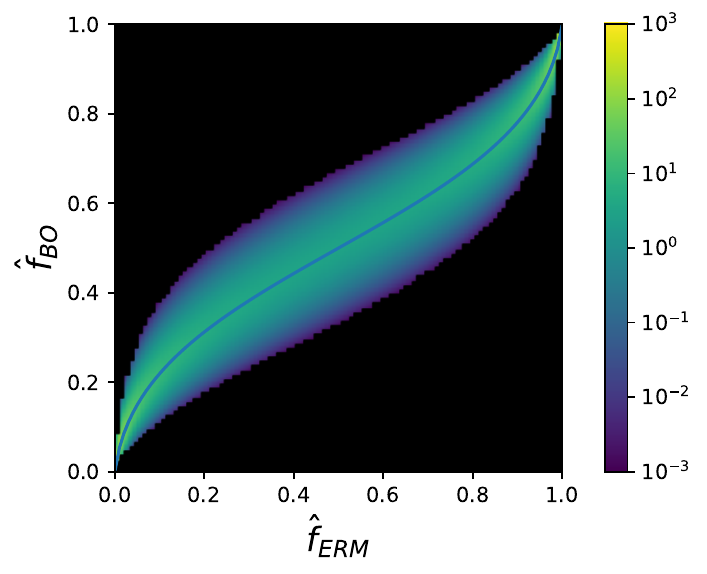}}
    \caption{The probability density $\rho_{\erm, \star}$ (left panel) and $\rho_{\erm, \bo}$ (right panel), at $\alpha = 10$, $\noisestr = 0.5$ and $\lambda = 0^{+}$. The blue curves are the mean of the marginal distribution of $f_{\star}$ and $\hat{f}_{\bo}$ respectively under fixed $\hat{f}_{\erm}$ , which are equal to $p - \Delta_p$ and $p - \tilde{\Delta}_p$. We observe overconfidence of the logistic classifier for these parameters. Test error of ERM is here $\varepsilon^{\erm}_g = 0.174$, very close to the one of BO $\varepsilon^{\rm bo}_g = 0.173$.}
    \label{fig:erm_joint_density_lambda_0}
\end{figure}

We now investigate the calibration as a function of the sample complexity $\alpha$. The plot (a) of Figure \ref{fig:lambda_0_calibration} shows the curve $\Delta_p$ at $\lambda = 0^{+}$ computed using the analytical expression \eqref{eq:thm_cal_erm}. The curve is compared to the numerical estimation of $\Delta_p$ (green crosses) and $\tilde{\Delta}_p$ (orange crosses). For a small $\dd{p}$, If we define $I_{p, \dd{p}} = \{ 1 \leqslant i \leqslant n_{\rm test} |  \hat{f}_{\erm}(x_i) \in [p, p + \dd{p}] \}$, $\Delta_p$ and $\tilde{\Delta}_p$ are estimated experimentally with the formulas
\begin{equation}
    \Delta_p \simeq p - \frac{ \sum_{i \in I_{p, \dd{p}}} f_{\star}(x_i)}{|I_{p, \dd{p}}|}, 
    \Tilde{\Delta}_p \simeq p - \frac{ \sum_{i \in I_{p, \dd{p}}} \hat{f}_{\bo}(x_i)}{|I_{p, \dd{p}}|}
\end{equation}
The calibrations $\Delta_p$ and $\tilde{\Delta}_p$ are both equal to the theoretical curve, further confirming the results of Equation \eqref{eq:thm_cal_erm_bo}. Note the transition at $\alpha_c \sim 2.4$: for $\alpha < \alpha_c$, the training data is linearly separable. Since $\lambda = 0^{+}$, the empirical risk has no minimum and the estimator $\vec{w}_{\erm}$ diverges in norm. As a consequence, $\Delta_p \rightarrow p - 0.5$, as we observe on the plot.
In the inset of Fig.~\ref{fig:lambda_0_calibration} (left) we depict the theoretical curve evaluate up to larger values of $\alpha$. We see a saturation at about $\Delta_p \simeq 0.0011 \neq 0$. We note that in the work of \cite{bai_dont_2021} (partly reproduced in Appendix \ref{app:logit}) the calibration was observed to go to $0$  as $1/\alpha$. 
This difference is due to the mismatch between the function producing the data (probit) and the estimator (logit) in our case (whereas \cite{bai_dont_2021} used logit for both) which will generically be present in real data and thus the decay to zero observed in \cite{bai_dont_2021} is not expected to be seen generically.

Right panel of Figure \ref{fig:lambda_0_calibration} displays the variance of $f_{\star}$ and $\hat{f}_{\bo}$ at fixed $\hat{f}_{\erm}$ as a function of $\alpha$. This plot illustrates that the conditional variance of $f_{\star}$ is significantly higher than that of $\hat{f}_{\bo}$, as was previously noted in figure \ref{fig:erm_joint_density_lambda_0}. 
This shows that the (non-regularized) logistic uncertainty captures rather decently the uncertainty due to limited number of samples.

\begin{figure}[t!]
    \centering
    \subfigure[]{
    \includegraphics[width=0.47\columnwidth]{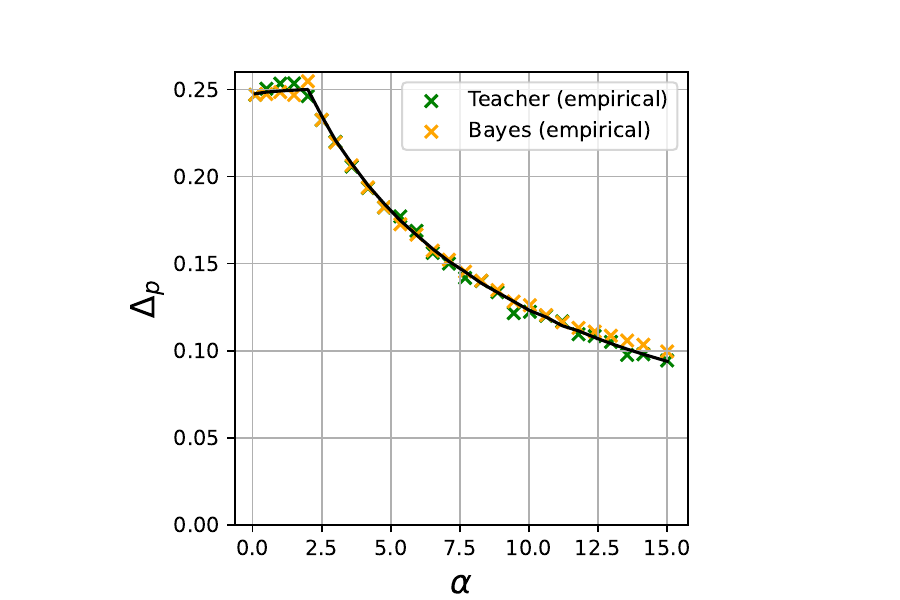}}
        \begin{picture}(0,0)
        \put(-175, 25){\includegraphics[height=1.5cm]{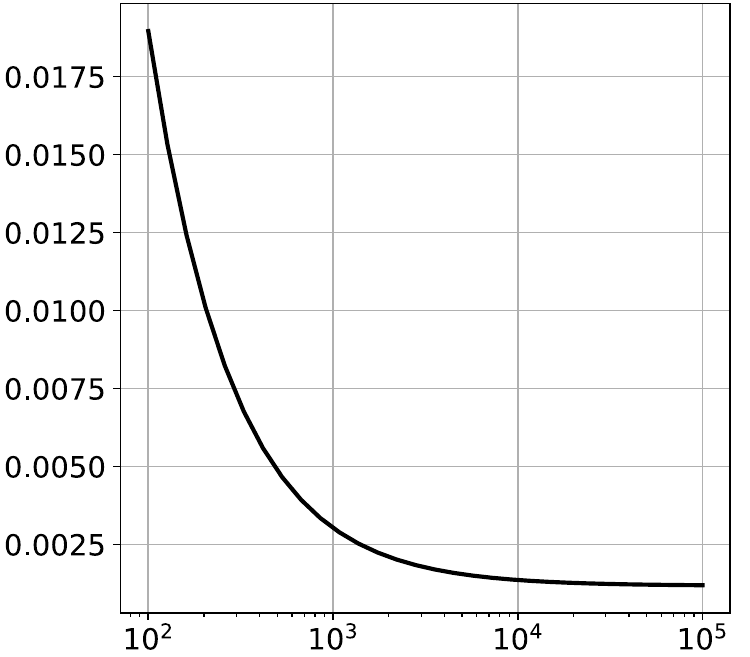}}
        \end{picture}
    \subfigure[]{\includegraphics[width=0.49\columnwidth]{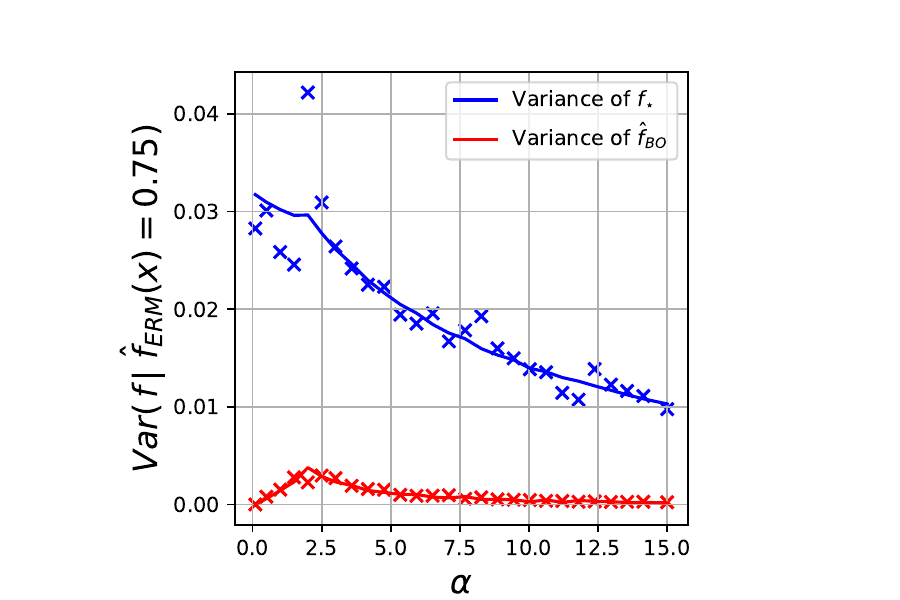}}
    \caption{(a) Calibration of the logistic regression with $\lambda = 0^{+}, \noisestr = 2, p=0.75$. Orange (respectively green) crosses are numerical estimation of $\tilde{\Delta}_p$ (respectively $\Delta_p$). Numerical values are obtained by averaging the calibration over $10$ test sets of size $n_{\rm test}=10^5$, at $d = 300$. Inset depicts the larger $\alpha$ behaviour. 
(b) 
Variance of $f_{\star}$ and $\hat{f}_{\bo}$ conditioned on $\prederm = p = 0.75$. Crosses are numerical values with the same parameters as figure (a). Though both $f_{\star}$ and $\hat{f}_{BO}$ have the same mean, their variance are significantly different.}
    \label{fig:lambda_0_calibration}
\end{figure}

\begin{figure}[t!]
    \centering
    \def\figwidth{\columnwidth}
    \def\figheight{\columnwidth}
    
    \subfigure[]{\includegraphics[width=0.47\columnwidth]{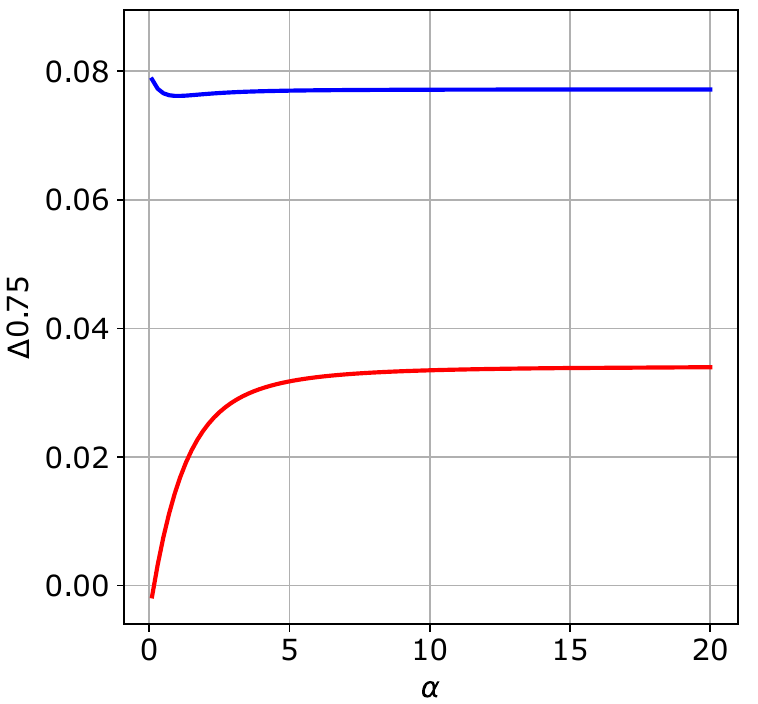}}
    \subfigure[]{\includegraphics[width=0.47\columnwidth]{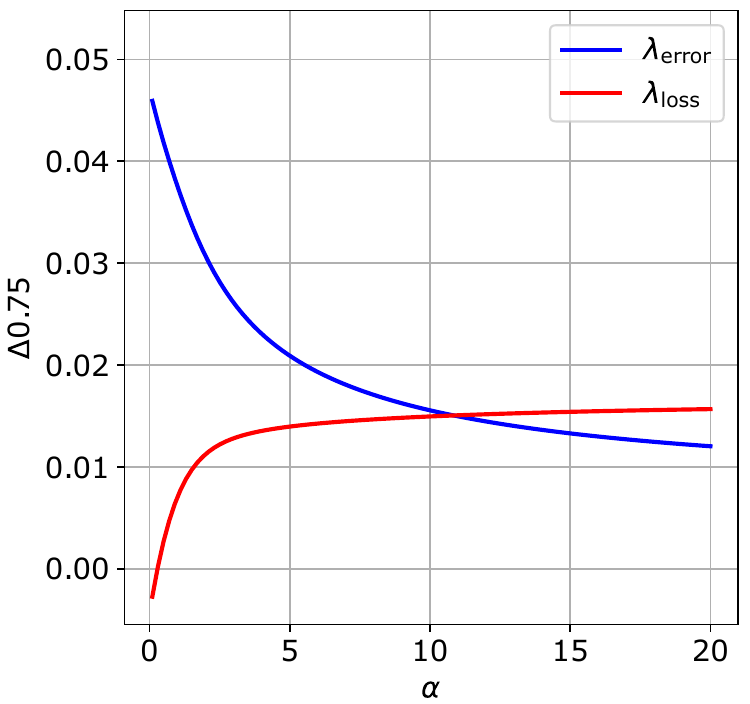}}
    \subfigure[]{\includegraphics[width=0.495\columnwidth]{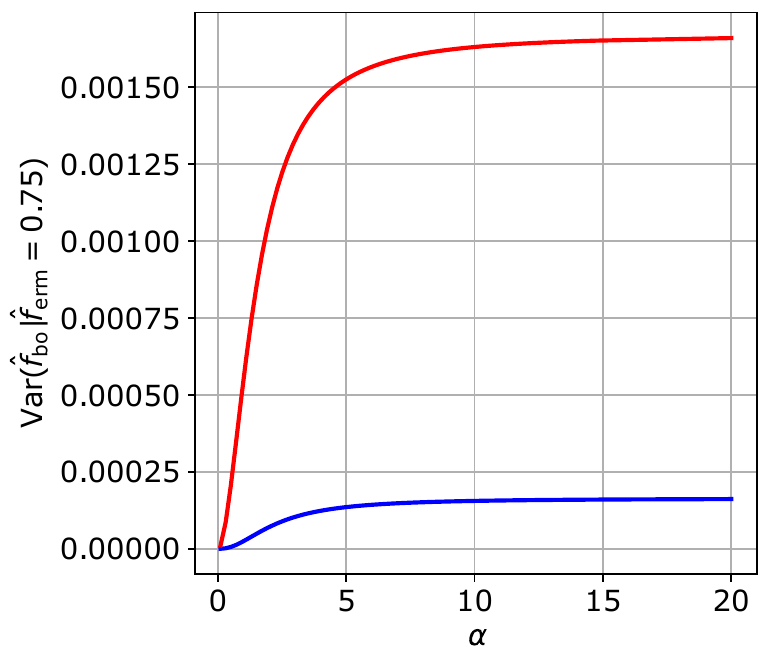}}
    \subfigure[]{\includegraphics[width=0.495\columnwidth]{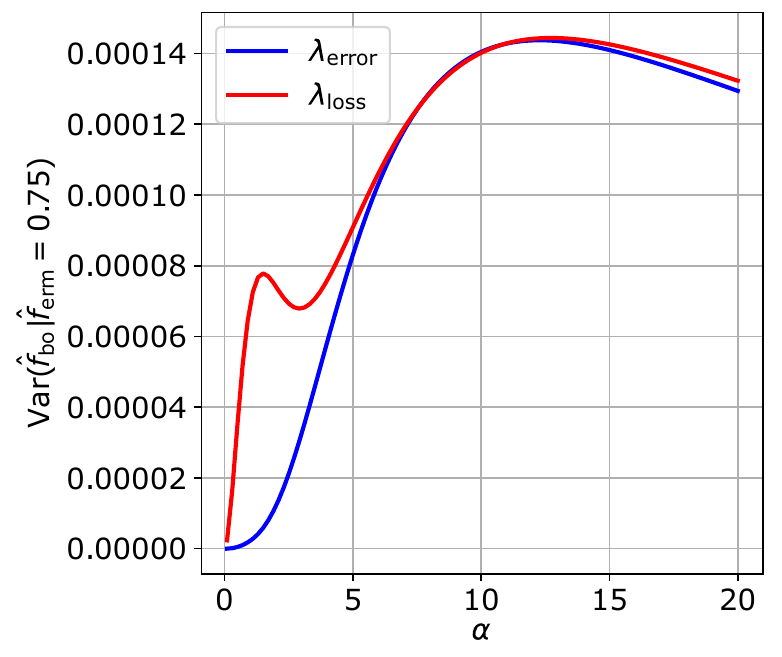}}
    \caption{(Top) The calibration $\Delta_{0.75}(\hat{f}_{\erm})$ as a function of $\alpha$  with $\lambda = \lambdaerror(\alpha, \noisestr)$ (blue curve) and $\lambda = \lambdaloss(\alpha, \noisestr)$ (red curve)s. (Bottom) Variance of $\hat{f}_{\bo}$ conditioned on $\hat{f}_{\erm}(\vec{x}) = 0.75$ with $\lambdaerror$ and $\lambdaloss$. In (a) and (c), $\tau = 0$ ; in (b) and (d), $\tau = 0.5$.
    }
    \label{fig:calibration_optimal_lambda}
    \vspace{-0.5cm}
\end{figure}


\begin{figure}[t!]
    \centering
    \subfigure[$\lambdaerror = 0.0976$,  $\varepsilon_g = 0.1732 $]{\includegraphics[width=0.44\columnwidth]{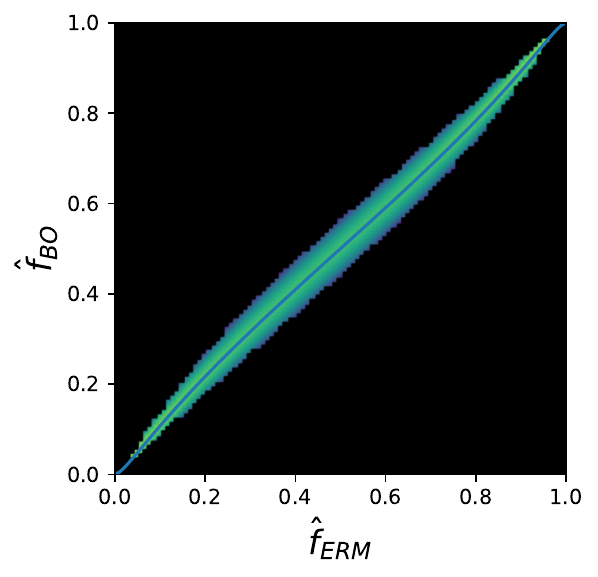}}
    \subfigure[$\lambdaloss = 0.0980$,  $\varepsilon_g = 0.1734 $]{\includegraphics[width=0.52\columnwidth]{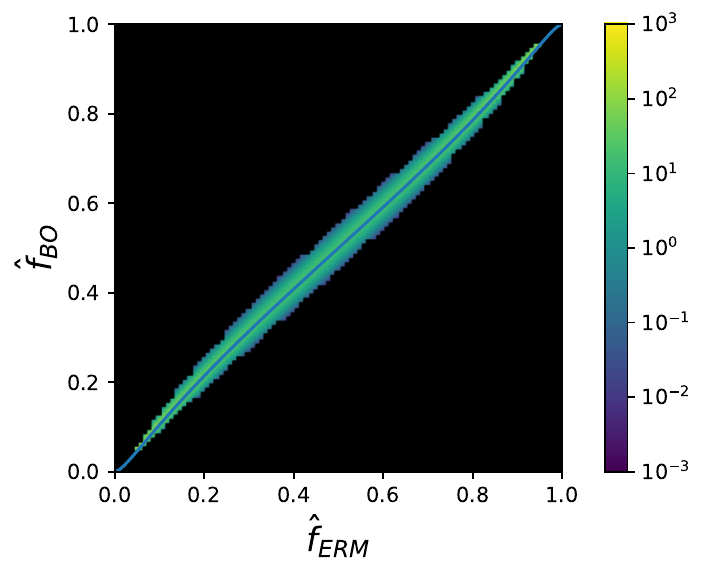}}
    \subfigure[$\lambdaerror = 0.0039$,  $\varepsilon_g = 0.0843$]{\includegraphics[width=0.44\columnwidth]{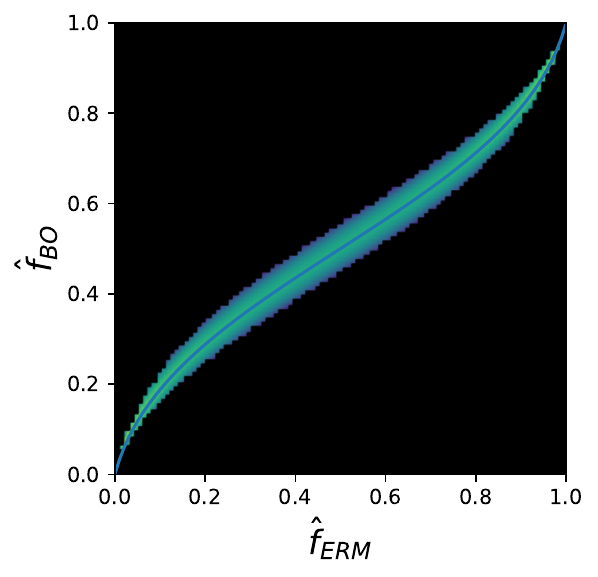}}
    \subfigure[$\lambdaloss=0.0096$,  $\varepsilon_g = 0.0847$]{\includegraphics[width=0.52\columnwidth]{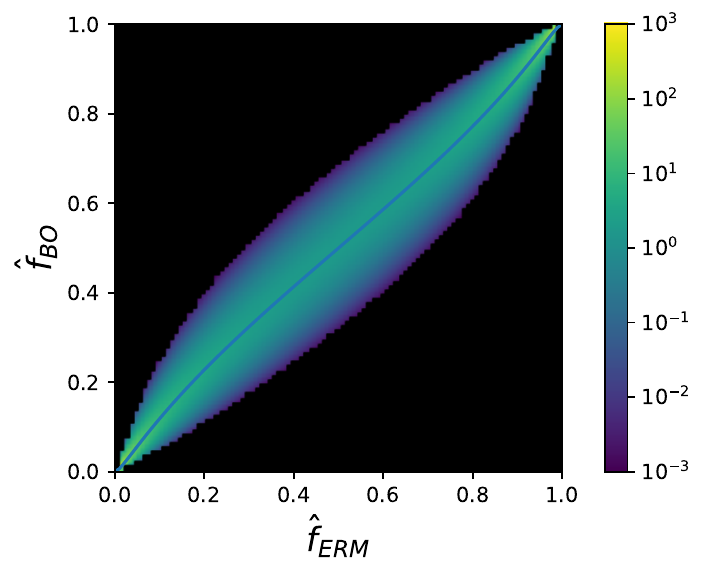}}
    \caption{Density $\rho_{\erm, \bo}$ for different $\alpha, \noisestr$. Top row: $\alpha = 10, \noisestr = 0.5$. Bayes test error is $\varepsilon_g^{\bo} = 0.1731$. Plot (a) (respectively (b)) is done at $\lambda = \lambdaerror$ (respectively $\lambda = \lambdaloss$). Bottom row: $\alpha = 5, \noisestr = 0$, $\varepsilon_g^{\bo} = 0.0839$. Plot (c) (respectively (d)) is done at $\lambda = \lambdaerror$ (respectively $\lambda = \lambdaloss$). On the bottom row, we can clearly see that the calibration is better for $\lambdaloss$. Generalization errors of ERM as well as the values of the regularizations are indicated below the plots.}
    \label{fig:density_alpha_10_optimal_lambda}
\end{figure}

    

\subsection{Effect of regularization on uncertainty and calibration} 
\label{sec:choosing_lambda_optimally}

Logistic regression is rarely used in practice without regularization. In Figs.~\ref{fig:multiple_densities_lambda} and \ref{fig:calibration_lambda} in appendix \ref{sec:additional_figures} we depict the role of regularization on the density $\rho_{\rm erm,bo}$. As one would anticipate as the regularization strength grows the overconfidence of the logistic classifier at small $\lambda$ becomes under-confidence at large $\lambda$.

One usually optimizes the strength $\lambda$ of the $\ell_2$ penalty through cross-validation by minimizing the validation error. Ideally, we would choose $\lambda$ that gives a low validation error but also that yields a well-calibrated estimator. We will denote $\lambdaerror$ (respectively $\lambdaloss$) the parameter that minimises the $0/1$ classification error (respectively the logistic loss) on the validation set. 
In the setting of the present paper, these two values of regularisation lead to a very close test error/loss. In other words, choosing one or another of these $\lambda$ seems to have little effect on the test performance of logistic regression. 

Figure \ref{fig:calibration_optimal_lambda} plots the calibration $\Delta_p$ in the noiseless (left panel) and noisy (right panel) settings. We observe that for most parameters ERM with $\lambdaloss$ is significantly less overconfident than with $\lambdaerror$. However, for larger values of $\alpha$ and $\tau$ we observe the opposite. 
We also note that for small $\alpha$ the logistic regression at $\lambdaloss$ even gets mildly underconfident, $\Delta_p<0$. The bottom panels of the figure depict the corresponding variance. Interestingly we see that in both cases, despite a better calibration, $\lambdaloss$ yields a higher variance than $\lambdaerror$ hence its point-wise estimates of uncertainty are not necessarily better.

Figure \ref{fig:density_alpha_10_optimal_lambda} shows 
$\rho_{\bo, \erm}$ evaluated at $\lambdaerror$ and $\lambdaloss$. Comparing the upper panels to Figure~\ref{fig:erm_joint_density_lambda_0} (at $\lambda=0$), it is clear that choosing $\lambda$ to optimize the error (and the loss) improves calibration. In the lower panels of Figure \ref{fig:density_alpha_10_optimal_lambda} we can also see that the calibration at $\lambdaloss$ (right panel) is better, i.e. the blue line is closer to $y=x$, than the one at $\lambdaerror$ (left panel). We conclude that using optimal regularization is clearly advantageous to obtain better calibrated classification. However, we also note that the interplay between the mean of the distribution (the calibration) and its variance is subtle and more investigation is needed into designing a model-agnostic method where both are optimal simultaneously.    

\section{Discussion}
This paper leverages on the properties of the GAMP algorithm and associated closed-form control of the posterior marginals to provide detailed theoretical analysis of uncertainty in a simple probit model. We investigate the relations between the respective uncertainties of the oracle, Bayes and regularized logistic regression. We see this as a grounding step for a line of future work that will leverage recent extensions of the GAMP algorithm and its associated analysis to multi-layer neural networks \cite{Aubin_2019, gerbelot2021graphbased}, learning with random features and kernels \cite{mei2019generalization,gerace2020generalisation,dhifallah2020precise}, estimation under generative priors \cite{9240945, pmlr-v107-aubin20a}, classification on more realistic models of data \cite{goldt2020modeling, goldt2021gaussian, seddik2020random}, etc. 
The present methodology is not restricted to classification and can be used for more thorough study of confidence intervals in high-dimensional regression, extending \cite{bai2021understanding}.  This is left for further studies. The code of this project is available at \url{https://github.com/lclarte/uncertainty-project}.

\paragraph{Acknowledgements--}
We thank C\'edric Gerbelot for useful discussion and Benjamin Aubin for his help on the numerical experiments. We acknowledge funding from the ERC under the European Union’s Horizon 2020 Research and Innovation Programme Grant Agreement 714608-SMiLe.

\printbibliography

\newpage 

\balance

\newpage
\appendix
\onecolumn

\section{Cavity derivation of the analytical results}
\label{sec:app:cavity}
In this appendix we sketch how the self-consistent equations \eqref{def:eq-bo} and \eqref{def:eq-erm-hats} characterizing the sufficient statistics $(q_{\bo}, m, q_{\erm})$ can actually be derived via the heuristic cavity method \cite{mezard1987spin,mezard2009information} from statistical physics. 

We shall use the notation of Rangan's GAMP algorithm \cite{rangan2011generalized} and present our results as a derivation of GAMP algorithm from cavity, or beleif propagation, as in \cite{zdeborova2016statistical}. This allows to connect all our results as well as the  state evolution equations of the GAMP \ref{alg:gamp} algorithm in a single framework. Note that in its most general form, GAMP can be used both as an algorithm for estimating the marginals of the posterior distribution $\vec{w}_{\amp} = \mathbb{E}[w|\mathcal{D}]$ or to minimize the empirical risk in \ref{eq:def:erm} - the only difference between the two being the choice of denoising functions $(f_{{\rm out}}, f_{w})$. 

The novelty of our approach consists of running two GAMP algorithms in parallel \emph{on the same instance} of data $\mathcal{D} = \{(\vec{x}^{\mu},y^{\mu})\}_{\mu=1}^{n}$ drawn from the probit model introduced in eq.~\eqref{eq:def:data}. Although we run the two versions of GAMP independently, they are correlated through the data $\mathcal{D}$ - and our goal is to characterize exactly their joint distribution.  

\subsection{Joint state evolution} 
Consider we are running two AMPs in parallel, one for BO estimation and one for ERM. To distinguish both messages, we will denote ERM messages with a tilde: $\tilde{\omega}^{t}, \tilde{V}^{t}$, etc. To derive the asymptotic distribution of the estimators $(\hat{\vec{w}}_{\amp}, \hat{\vec{w}}_{\erm})$, it is more convenient to start from a close cousin of AMP: the reduced Belief Propagation equations (rBP). Note that in the high-dimensional limit that we are interested in this manuscript, rBP is equivalent to AMP, see for instance \cite{Aubin_2019} or \cite{9240945} for a detailed derivation. Written in coordinates, the rBP equations are given by:
\begin{align}
	&\begin{cases}
		\omega^{t}_{\mu\to i} = \sum\limits_{j\neq i} x^{\mu}_{j}\hat{w}_{j\to\mu}^{t}\\
		V^{t}_{\mu\to i} = \sum\limits_{j\neq i}(x^{\mu}_{j})^2 \hat{c}^{t}_{j\to \mu}
	\end{cases}, &&
	\begin{cases}
		g^{t}_{\mu\to i} = f_{\text{out}}(y^{\mu}, \omega_{\mu\to i}^{t}, V^{t}_{\mu\to i})\\
		\partial g^{t}_{\mu\to i} = \partial_{\omega}f_{\text{out}}(y^{\mu}, \omega_{\mu\to i}^{t}, V^{t}_{\mu\to i})\\
	\end{cases}\\
	&\begin{cases}
		b^{t}_{\mu\to i} = \sum\limits_{\nu\neq \mu} x^{\nu}_{i}g^{t}_{\nu\to i}\\
		A^{t}_{\mu\to i} = -\sum\limits_{\nu\neq \mu} (x^{\nu}_{i})^2 \partial g^{t}_{\nu\to i}\\
	\end{cases},	&&
	\begin{cases}
		\hat{w}^{t+1}_{i\to\mu} f_{w}(b^{t}_{i\to\mu}, A^{t}_{i\to \mu})	\\
		\hat{c}^{t+1}_{i\to\mu} \partial_{b}f_{w}(b^{t}_{\mu\to i}, A^{t}_{\mu\to i})	
	\end{cases}
\end{align}
\noindent where $(f_{\text{out}}, f_{w})$ denote the denoising functions that could be associated either to BO or ERM estimation, and that can be generically written in terms of an estimation likelihood $P_{\text{out}}$ and prior $P_{w}$ as:
\begin{align}
\begin{cases}
f_{\text{out}}(y,\omega, V) &= \partial_{\omega}\log \mathcal{Z}_{\text{out}}(y,\omega, V) \\
\mathcal{Z}_{\text{out}}(y,\omega, V) &= \int_{\mathbb{R}}\frac{\dd x}{\sqrt{2\pi V}}e^{-\frac{(x-\omega)^2}{2V}} P_{\text{\out}}(y|x)
\end{cases},&&
\begin{cases}
    f_{w}(b, A) &= \partial_{b}\log \mathcal{Z}_{w}(b,A)\\
    \mathcal{Z}_{w}(b, A) &= \int_{\mathbb{R}}\dd w~P_{w}(w) e^{-\frac{1}{2}Aw^2+bw}
\end{cases}.
\end{align}
By assumption, the rBP messages are independent from each other, and since we are running both BO and ERM independently, they are only coupled to each other through the data, which has been generated by the same data model:
\begin{align}
    y^{\mu} \sim P_{0}(\cdot|\wstar^{\top}\vec{x}^{\mu}), && \vec{x}^{\mu}\sim\mathcal{N}(0, \sfrac{1}{d}\mat{I}_{d}),&& \wstar\sim \prod\limits_{i=1}^{d}P_{0}(w_{\star i}).
\end{align}
Note that here we work in a more general setting than the one in the main manuscript \eqref{eq:def:data}. Indeed, the derivation presented here work for \emph{any} factorised distribution of teacher weights $\wstar$ and any likelihood $P_{0}$ (of which the probit is a particular case). For convenience, define the so-called \emph{teacher local field}:
\begin{align}
z_{\mu} = \sum\limits_{j=1}^{d}x^{\mu}_{j}w_{\star j}	
\end{align}

\subsection*{Step 1: Asymptotic joint distribution of $(z_{\mu}, \omega_{\mu\to i}^{t}, \tilde{\omega}_{\mu\to i}^{t})$}
Note that $(z_{\mu}, \omega_{\mu\to i}^{t}, \tilde{\omega}_{\mu\to i}^{t})$ are given by a sum of independent random variables with variance $d^{-1/2}$, and therefore by the Central Limit Theorem in the limit $d\to\infty$ they are asymptotically Gaussian. Therefore we only need to compute their means, variances and cross correlation. The means are straightforward, since $x^{\mu}_{i}$ have mean zero and therefore they will also have mean zero. The variances are given by:
\begin{align}
\mathbb{E}\left[z_{\mu}^2\right] &= \mathbb{E}\left[\sum\limits_{i=1}^{d}\sum\limits_{j=1}^{d}x^{\mu}_{i}x^{\mu}_{j}w_{\star i}w_{\star j}\right] = \sum\limits_{i=1}^{d}\sum\limits_{j=1}^{d}\mathbb{E}\left[x^{\mu}_{i}x^{\mu}_{j}\right]
w_{\star i}w_{\star j}\notag = \frac{1}{d}\sum\limits_{i=1}^{d}\sum\limits_{j=1}^{d}\delta_{ij} w_{\star i}w_{\star j}\notag\\
 &= \frac{||\wstar||^2_{2}}{d} \underset{d\to\infty}{\rightarrow} \rho \\
\mathbb{E}\left[\left(\omega^{t}_{\mu\to i}\right)^2\right] &= \mathbb{E}\left[\sum\limits_{j\neq i}^{d}\sum\limits_{k\neq i}^{d}x^{\mu}_{j}x^{\mu}_{k}\hat{w}^{t}_{j\to\mu}\hat{w}^{t}_{k\to\mu}\right] = \sum\limits_{j\neq i}^{d}\sum\limits_{k\neq i}^{d}\mathbb{E}\left[x^{\mu}_{j}x^{\mu}_{k}\right]
\hat{w}^{t}_{j\to\mu}\hat{w}^{t}_{k\to\mu} \notag\\
&= \frac{1}{d}\sum\limits_{j\neq i}^{d}\sum\limits_{k\neq i}^{d}\delta_{jk}\hat{w}^{t}_{j\to\mu}\hat{w}^{t}_{k\to\mu}
=\frac{1}{d}\sum\limits_{j\neq i}^{d}\left(\hat{w}^{t}_{j\to\mu}\right)^2 = \frac{||\hat{\vec{w}}^{t}||^{2}_{2}}{d}-\frac{1}{d}(\hat{w}^{t}_{i\to \mu})^2 \underset{d\to\infty}{\rightarrow} q^{t}\\
\mathbb{E}\left[z_{\mu}\omega^{t}_{\mu\to i}\right] &= \mathbb{E}\left[\sum\limits_{j\neq i}^{d}\sum\limits_{k=1}^{d}x^{\mu}_{j}x^{\mu}_{k}\hat{w}^{t}_{j\to\mu}w_{\star k}\right] = \sum\limits_{j\neq i}^{d}\sum\limits_{k=1}^{d}\mathbb{E}\left[x^{\mu}_{j}x^{\mu}_{k}\right]
\hat{w}^{t}_{j\to\mu}w_{\star k}\notag \\
&= \frac{1}{d}\sum\limits_{j\neq i}^{d}\sum\limits_{k=1}^{d}\delta_{jk}\hat{w}^{t}_{j\to\mu}w_{\star k}
=\frac{1}{d}\sum\limits_{j\neq i}^{d}\hat{w}^{t}_{j\to\mu}w_{\star j} = \frac{\hat{\vec{w}}^{t}\cdot \wstar}{d}-\frac{1}{d}\hat{w}^{t}_{i\to \mu}w_{\star i} \underset{d\to\infty}{\rightarrow} m^{t} \\
\mathbb{E}\left[\omega^{t}_{\mu\to i}\tilde{\omega}^{t}_{\mu\to i}\right] &= \mathbb{E}\left[\sum\limits_{j\neq i}^{d}\sum\limits_{k\neq i}^{d}x^{\mu}_{j}x^{\mu}_{k}\hat{w}^{t}_{j\to\mu}\tilde{\hat{w}}^{t}_{k\to\mu}\right] \notag \\
&= \sum\limits_{j\neq i}^{d}\sum\limits_{k\neq i}^{d}\mathbb{E}\left[x^{\mu}_{j}x^{\mu}_{k}\right]
\hat{w}^{t}_{j\to\mu}\tilde{\hat{w}}^{t}_{k\to\mu} = \frac{1}{d}\sum\limits_{j\neq i}^{d}\sum\limits_{k\neq i}^{d}\delta_{jk}\hat{w}^{t}_{j\to\mu}\tilde{\hat{w}}^{t}_{k\to\mu}\notag \\
&=\frac{1}{d}\sum\limits_{j\neq i}^{d}\hat{w}^{t}_{j\to\mu}\tilde{\hat{w}}^{t}_{j\to\mu} = \frac{\hat{\vec{w}}^{t}\cdot \tilde{\hat{\vec{w}}}^{t}}{d}-\frac{1}{d}\hat{w}^{t}_{i\to \mu}\tilde{\hat{w}}^{t}_{i\to \mu} \underset{d\to\infty}{\rightarrow} Q^{t}
\end{align}
\noindent where we have used that $\hat{w}_{i\to\mu}^{t} = O(d^{-1/2})$ to simplify the sums at large $d$. Summarising our findings:
\begin{align}
\label{eq:joint:fields}
(z_{\mu}, \omega^{t}_{\mu\to i}, \tilde{\omega}_{\mu\to i}^{t})\sim\mathcal{N}\left(\vec{0}_{3},\begin{bmatrix}
\rho & m^{t} & \tilde{m}^{t} \\ m^{t} & q^{t} & Q^{t} \\ \tilde{m}^{t} & Q^{t} & \tilde{q}^{t}	
\end{bmatrix}
\right)	
\end{align}
\noindent with:
\begin{align}
\rho \equiv \frac{1}{d}||\wstar||^2	, && q^{t} \equiv \frac{1}{d}||\hat{\vec{w}}_{\text{BO}}^{t}||^2, && \tilde{q}^{t} \equiv \frac{1}{d}||\hat{\vec{w}}_{\text{ERM}}^{t}||^2 \notag\\ 
m^{t} \equiv \frac{1}{d} \hat{\vec{w}}_{\text{BO}}\cdot \wstar, && \tilde{m}^{t} \equiv \frac{1}{d} \hat{\vec{w}}_{\text{ERM}}\cdot \wstar, && Q^{t} \equiv \frac{1}{d} \hat{\vec{w}}_{\text{BO}}\cdot \hat{\vec{w}}_{\text{ERM}}
\end{align}

\subsection*{Step 2: Concentration of variances $V_{\mu\to i}^{t}, \tilde{V}_{\mu\to i}^{t}$}
Since the variances $V_{\mu\to i}^{t}, \tilde{V}_{\mu\to i}^{t}$ depend on $(x_{i}^{\mu})^2$, in the asymptotic limit $d\to\infty$ they concentrate around their means:
\begin{align}
\mathbb{E}\left[V^{t}_{\mu\to i}\right]  = \sum\limits_{j\neq i}\mathbb{E}\left[\left(x^{\mu}_{i}\right)^2\right]\hat{c}^{t}_{j\to \mu}	= \frac{1}{d}\sum\limits_{j\neq i}\hat{c}^{t}_{j\to \mu} =  \frac{1}{d}\sum\limits_{j=1}^{d}\hat{c}^{t}_{j\to \mu} - \frac{1}{d}\hat{c}^{t}_{i\to \mu}\underset{d\to\infty}{\rightarrow} V^{t} \equiv \frac{1}{d}\sum\limits_{j=1}^{d}\hat{c}^{t}_{j}
\end{align}
\noindent where we have defined the variance overlap $V^{t}$. The same argument can be used for $\tilde{V}^{t}_{\mu\to in}$. Summarising, asymptotically we have:
\begin{align}
V_{\mu\to i}^{t} \to V^{t}, && \tilde{V}_{\mu\to i}^{t} \to \tilde{V}^{t}
\end{align}

\subsection*{Step 3: Distribution of $b_{\mu\to i}^{t}, \tilde{b}_{\mu\to i}^{t}$}
By definition, we have
\begin{align}
	b^{t}_{\mu\to i} &= \sum\limits_{\nu\neq \mu} x^{\nu}_{i}g^{t}_{\nu\to i} = \sum\limits_{\nu\neq \mu} x^{\nu}_{i}f_{\text{out}}(y^{\mu}, \omega_{\nu\to i}^{t}, V^{t}_{\nu\to i}) = \sum\limits_{\nu\neq \mu} x^{\nu}_{i}f_{\text{out}}(f_{0}(z_{\nu} + \noisestr \xi_{\nu}), \omega_{\nu\to i}^{t}, V^{t}_{\nu\to i})\notag\\
\end{align}
Note that in the sum $z_{\nu} = \sum\limits_{j=1}^{d}x^{\nu}_{j}w_{\star j}$ there is a term $i=j$, and therefore $z_{\mu}$ is correlated with $x^{\nu}_{i}$. To make this explicit, we split the teacher local field: 
\begin{align}
z_{\mu} = \sum\limits_{j=1}^{d}x^{\mu}_{j}w_{\star j}	 = \underbrace{\sum\limits_{j\neq i}x^{\mu}_{j}w_{\star j}}_{z_{\mu\to i}}+x^{\mu}_{i}w_{\star i}
\end{align}
\noindent and note that $z_{\mu\to i} = O(1)$ is independent from $x^{\nu}_{i}$. Since $x^{\mu}_{i}w_{\star i} = O(d^{-1/2})$, to take the average at leading order, we can expand the denoising function:
\begin{align}
f_{\text{out}}(f_{0}(z_\mu + \noisestr \xi_{\mu}), \omega^{t}_{\nu\to i}, V^{t}_{\nu\to i}) &= f_{\text{out}}(f_{0}(z_{\nu\to i} + \noisestr \xi_{\nu}), \omega^{t}_{\nu\to i}, V^{t}_{\nu\to i}) \\ &+ \partial_{z}f_{\text{out}}(f_{0}(z_{\nu\to i} + \noisestr \xi_{\nu}), \omega^{t}_{\nu\to i}, V^{t}_{\nu\to i})x^{\nu}_{i}w_{\star i} + O(d^{-1})\notag
\end{align}
Inserting in the expression for $b^{t}_{\mu\to i}$,
\begin{align}
	b^{t}_{\mu\to i} &= \sum\limits_{\nu\neq \mu} x^{\nu}_{i}f_{\text{out}}(f_{0}(z_{\nu\to i} + \noisestr \xi_{\nu}), \omega_{\nu\to i}^{t}, V^{t}_{\nu\to i})
	\\& + \sum\limits_{\nu \neq \mu} (x^{\nu}_{i})^2 \partial_{z}f_{\text{out}}(f_{0}(z_{\nu\to i} + \noisestr \xi_{\nu}), \omega_{\nu\to i}^{t}, V^{t}_{\nu\to i})w_{\star i} + O(d^{-3/2})\notag
\end{align}
Therefore:
\begin{align}
\mathbb{E}\left[b^{t}_{\mu\to i}\right] &= \frac{w_{\star i}}{d}\sum\limits_{\nu\neq\mu} \partial_{z}f_{\text{out}}(f_{0}(z_{\nu\to i} + \noisestr \xi_{\nu}), \omega_{\nu\to i}^{t}, V^{t}_{\nu\to i})+ O(d^{-3/2}) \notag\\
&= \frac{w_{\star i}}{d}\sum\limits_{\nu=1}^{n} \partial_{z}f_{\text{out}}(f_{0}(z_{\nu\to i} + \noisestr \xi_{\nu}), \omega_{\nu\to i}^{t}, V^{t}_{\nu\to i})  + O(d^{-3/2})
\end{align}
Note that as $d\to\infty$, for fixed $t$ and for all $\nu$, the fields $(z_{\nu\to i}, \omega^{t}_{\nu\to i})$ are identically distributed according to average in eq.~\eqref{eq:joint:fields}. Therefore, 
\begin{align}
\frac{1}{d}\sum\limits_{\nu=1}^{n} \partial_{z}f_{\text{out}}(f_{0}(z_{\nu\to i} + \noisestr \xi_{\nu}), \omega_{\nu\to i}^{t}, V^{t}_{\nu\to i}) \underset{d\to\infty}{\rightarrow} \alpha~\mathbb{E}_{(\omega, z), \xi}\left[\partial_{z}f_{\text{out}}(f_{0}(z + \noisestr \xi), \omega, V^{t})\right]	\equiv \hat{m}^{t}
\end{align}
\noindent so:
\begin{align}
\mathbb{E}\left[b^{t}_{\mu\to i}\right]  \underset{d\to\infty}{\rightarrow} w_{\star i}\hat{m}^{t}.
\end{align}
Similarly, the variance is given by:
\begin{align}
&\text{Var}\left[b^{t}_{\mu\to i}\right]\\
&= \sum\limits_{\nu\neq\mu}\sum\limits_{\kappa\neq\mu}\mathbb{E}\left[x^{\nu}_{i}x^{\kappa}_{i}\right] f_{\text{out}}(f_{0}(z_{\nu\to i} + \noisestr \xi_{\nu}), \omega_{\nu\to i}^{t}, V^{t}_{\nu\to i}) f_{\text{out}}(f_{0}(z_{\kappa\to i} + \noisestr \xi_{\kappa}), \omega_{\kappa\to i}^{t}, V^{t}_{\kappa\to i})+ O(d^{-2})\notag\\
&= \frac{1}{d}\sum\limits_{\nu\neq\mu}f_{\text{out}}(f_{0}(z_{\nu\to i} + \noisestr \xi_{\nu}), \omega_{\nu\to i}^{t}, V^{t}_{\nu\to i})^2 + O(d^{-2})\notag\\
&= \frac{1}{d}\sum\limits_{\nu=1}^{n}f_{\text{out}}(f_{0}(z_{\nu\to i} + \noisestr \xi_{\nu}), \omega_{\nu\to i}^{t}, V^{t}_{\nu\to i})^2 + O(d^{-2}) \notag\\
&\underset{d\to\infty}{\rightarrow} \alpha ~\mathbb{E}_{(z,\omega), \xi}\left[f_{\text{out}}(f_{0}(z + \noisestr \xi), \omega, V^{t})^2\right]\equiv \hat{q}^{t}
\end{align}
The same discussion holds for the ERM. We now just need to compute the correlation between both fields:
\begin{align}
&\text{Cov}\left[b^{t}_{\mu\to i}, \tilde{b}^{t}_{\mu\to i}\right]\\
&= \sum\limits_{\nu\neq\mu}\sum\limits_{\kappa\neq\mu}\mathbb{E}\left[x^{\nu}_{i}x^{\kappa}_{i}\right] f_{\text{out}}(f_{0}(z_{\nu\to i} + \noisestr \xi_{\nu}), \omega_{\nu\to i}^{t}, V^{t}_{\nu\to i}) \tilde{f}_{\text{out}}(f_{0}(z_{\kappa\to i} + \noisestr \xi_{\kappa}), \tilde{\omega}_{\kappa\to i}^{t}, \tilde{V}^{t}_{\kappa\to i})+ O(d^{-2})\notag\\
&= \frac{1}{d}\sum\limits_{\nu=1}^{n}f_{\text{out}}(f_{0}(z_{\nu\to i} + \noisestr \xi_{\nu}), \omega_{\nu\to i}^{t}, V^{t}_{\nu\to i})\tilde{f}_{\text{out}}(f_{0}(z_{\nu\to i} + \noisestr \xi_{\nu}), \tilde{\omega}_{\nu\to i}^{t}, \tilde{V}^{t}_{\nu\to i}) + O(d^{-2})\notag\\
&\underset{d\to\infty}{\rightarrow} \alpha ~\mathbb{E}_{(z,\omega, \tilde{\omega}), \xi}\left[f_{\text{out}}(f_{0}(z + \noisestr \xi), \omega, V^{t})\tilde{f}_{\text{out}}(f_{0}(z + \noisestr \xi), \tilde{\omega}, \tilde{V}^{t})\right]\equiv \hat{Q}^{t}
\end{align}
To summarise, we have:
\begin{align}
(b^{t}_{\mu\to i}, \tilde{b}^{t}_{\mu\to i})	\sim\mathcal{N}\left(w_{\star i}\begin{bmatrix}\hat{m}^{t}\\ \tilde{\hat{m}}^{t} \end{bmatrix}, \begin{bmatrix}\hat{q}^{t} & \hat{Q}^{t} \\ \hat{Q}^{t} & \tilde{\hat{q}}^{t}\end{bmatrix}\right)
\end{align}

\subsection*{Step 4: Concentration of $A_{\mu\to i}^{t}, \tilde{A}_{\mu\to i}^{t}$}
The only missing piece is to determine the distribution of the prior variances $A_{\mu\to i}^{t}, \tilde{A}_{\mu\to i}^{t}$. Similar to the previous variance, they concentrate:
\begin{align}
	A_{\mu\to i}^{t} &= -\sum\limits_{\nu\neq \mu}(x^{\nu}_{i})^2 \partial_{\omega}f_{\text{out}}(y^{\nu}, \omega_{\nu\to i}^{t}, V^{t}_{\nu\to i}) \\
	&= -\sum\limits_{\nu\neq \mu}(x^{\nu}_{i})^2 \partial_{\omega}f_{\text{out}}(f_{0}(z_{\nu\to i} + \noisestr \xi_{\nu}), \omega_{\nu\to i}^{t}, V^{t}_{\nu\to i}) + O(d^{-3/2})\notag\\
	&= -\frac{1}{d}\sum\limits_{\nu = 1} \partial_{\omega}f_{\text{out}}(f_{0}(z_{\nu\to i} + \noisestr \xi_{\nu}), \omega_{\nu\to i}^{t}, V^{t}_{\nu\to i}) + O(d^{-3/2})\notag\\
	&\underset{d\to\infty}{\rightarrow} - \alpha ~\mathbb{E}_{(z,\omega), \xi}\left[\partial_{\omega}f_{\text{out}}(f_{0}(z + \noisestr \xi), \omega, V^{t})\right]\equiv \hat{V}^{t}
\end{align}

\subsection*{Summary}
We now have all the ingredients we need to characterise the asymptotic distribution of the estimators:
\begin{align}
\hat{\vec{w}}_{\text{BO}} &\sim f_{\text{out}}(\wstar\hat{m}^{t}+\sqrt{\hat{q}^{t}}\vec{\xi}, \hat{V}^{t})\\
\hat{\vec{w}}_{\text{ERM}} &\sim \tilde{f}_{\text{out}}(\wstar\tilde{\hat{m}}^{t}+\sqrt{\tilde{\hat{q}}^{t}}\vec{\eta}, \tilde{\hat{V}}^{t})
\end{align}
\noindent where $\vec{\eta}, \vec{\xi} \sim\mathcal{N}(\vec{0},\mat{I}_{d})$ are independent Gaussian variables. From that, we can recover the usual GAMP state evolution equations for the overlaps:
\begin{align}
	\begin{cases}
		V^{t+1} = \mathbb{E}_{(w_{\star}, \xi)}\left[\partial_{b}f_{w}(\hat{m}^{t}w_{\star}+\sqrt{\hat{q}^{t}}\xi, \hat{V}^{t})\right]\\
		q^{t+1} = \mathbb{E}_{(w_{\star}, \xi)}\left[f_{w}(\hat{m}^{t}w_{\star}+\sqrt{\hat{q}^{t}}\xi, \hat{V}^{t})^2\right]\\
		m^{t+1} = \mathbb{E}_{(w_{\star}, \xi)}\left[f_{w}(\hat{m}^{t}w_{\star}+\sqrt{\hat{q}^{t}}\xi, \hat{V}^{t})w_{\star i}\right]
	\end{cases}, && 	\begin{cases}
		\hat{V}^{t} =-\alpha \mathbb{E}_{(z, \omega), \xi}\left[\partial_{\omega}f_{\text{out}}(f_{0}(z + \noisestr \xi), \omega, V^{t})\right]\\
		\hat{q}^{t} =\alpha \mathbb{E}_{(z, \omega), \xi}\left[f_{\text{out}}(f_{0}(z + \noisestr \xi), \omega, V^{t})^2\right]\\
		\hat{m}^{t} =\alpha \mathbb{E}_{(z, \omega), \xi}\left[f_{\text{out}}(f_{0}(z + \noisestr \xi), \omega, V^{t})\right]
	\end{cases}
\end{align}
\noindent which is also valid for the tilde variables. But we can also get a set of equations for the correlations:
\begin{align}
\begin{cases}
	Q^{t} = \mathbb{E}_{w_{\star},(b, \tilde{b})}\left[f_{w}(b, \hat{V}^{t})\tilde{f}_{w}\left(\tilde{b}, \tilde{\hat{V}}^{t}\right)\right]\\
	\hat{Q}^{t} =\alpha \mathbb{E}_{(z,\omega, \tilde{\omega}), \xi}\left[f_{\text{out}}(f_{0}(z + \tau \xi), \omega, V^{t})\tilde{f}_{\text{out}}(f_{0}(z + \tau \xi), \tilde{\omega}, \tilde{V}^{t})\right]
\end{cases}	
\end{align}

\subsection{Simplifications}

\subsection*{Simplifying BO state evolution}
\label{sec:simplifying_bo}
State evolution of BO can be reduced to two equations. First, note that asymptotically
\begin{equation*}
    m := \frac{1}{d} \hat{\vec{w}}_{\bo} \cdot \wstar = \frac{1}{d} \mathbb{E}_{\wstar, \mathcal{D}} \left[ \hat{\vec{w}}_{\bo} \cdot \wstar\right] \\
\end{equation*}
with high probability. By Nishimori identity, the vector $\wstar$ in the expectation can be replaced by an independent copy of the Bayesian posterior. This yields: 
\begin{equation*}
    \frac{1}{d} \mathbb{E}_{\wstar, \mathcal{D}} \left[ \hat{\vec{w}}_{\bo} \cdot \wstar \right] = \frac{1}{d} \mathbb{E}_{\mathcal{D}} \left[ \hat{\vec{w}}_{\bo} \right] = q 
\end{equation*}
Hence $m = q$. Similarly, noting $\langle \cdot \rangle$ the average over the posterior distribution:  
\begin{equation*}
    V = \frac{1}{d} \langle \| \vec{w} - \hat{\vec{w}}_{\bo} \|^2 \rangle = \frac{1}{d} \mathbb{E}_{\mathcal{D}} \left[ \langle \| \vec{w} - \hat{\vec{w}}_{\bo} \|^2 \rangle \right] = \frac{1}{d} \mathbb{E}_{\mathcal{D}} \left[ \langle \| \vec{w} \|^2 \rangle \right] - \frac{1}{d} \mathbb{E}_{\mathcal{D}} \left[ \hat{\vec{w}}_{\bo} \cdot \hat{\vec{w}}_{\bo} \right]
\end{equation*}
Like before, we used the fact that in asymptotically, $\langle \| \vec{w} - \hat{\vec{w}}_{\bo} \|^2 \rangle$ concentrates around its mean. Using Nishimori, the first term is equal to $\mathbb{E}_{\wstar} \left[ \| \wstar \|^2 \right] = 1$. By definition, the second term is equal to $q$, thus $V = 1 - q$.

Using similar arguments, $\hat{m} = \hat{q} = \hat{V}$. Thus, the state evolution can be reduced to two equations on $q$ and $\hat{q}$. 

\subsection*{Simplifying the $\Q, \hatQ$ equations}
In fact, the Nishimori property also allow us to show that the cross-correlation $\Q, \hatQ$ are the same as the overlaps $\merm, \hatmerm$, in a similar way to \ref{sec:simplifying_bo}. Indeed,

\begin{equation}
    \Q = \frac{1}{d} \hat{\vec{w}}_{\bo} \cdot \hat{\vec{w}}_{\erm} = \frac{1}{d} \mathbb{E}_{\mathcal{D}} \left[ \hat{\vec{w}}_{\bo} \cdot \hat{\vec{w}}_{\erm} \right] = \frac{1}{d} \mathbb{E}_{\wstar, \mathcal{D}} \left[ \wstar \cdot \hat{\vec{w}}_{\erm} \right] = \tilde{m}
\end{equation}
Alternatively, we can also prove that directly showing that the iterations for $Q^{t}$ are a stable orbit of $\tilde{m}^{t}$. Indeed, assume that at time step $t$ we have $\Q^t = \merm^t$ and $\hatQ^t = \hatmerm^t$. Then, focusing at our specific setting, at time $t+1$ we have: 
\begin{align*}
    \Q^{t+1} &= \mathbb{E}_{w_{\star}, b, \tilde{b}}[ f_w(b, \hatVbo) f_w(b, \hat{V}) ] = \mathbb{E}_{w_{\star}}\left[ \frac{b}{\hatVbo + 1}\frac{\tilde{b}}{\hatVerm + \lambda}\right] = \mathbb{E}_{w_{\star}}\left[\frac{\hatQ + \hatmerm \hatmbo}{(\hatVbo + 1)(\hatVerm + \lambda)}\right] \\
    &= \mathbb{E}_{w_{\star}}\left[\frac{\hatmerm}{\hatVerm + \lambda}\right].
\end{align*}
Because as we have shown above $\hatmbo = \hatqbo$ and $\hatQ^t = \hatmerm^t$. This is precisely the equation for $\merm$.

\subsection{Evaluating the equations}
\subsection*{Bayes-optimal}
In Bayes-optimal estimation, the estimation likelihood $P_{\text{out}}$ and prior $P_{w}$ match exactly that of the generating model for data, which for the model \eqref{eq:def:data} is:
\begin{align}
    P_{\text{out}}(y|x) = \frac{1}{2}\erfc\left(-\frac{y\omega}{\sqrt{2\Delta}}\right), && P_{w}(w) = \mathcal{N}(0,1).
\end{align}
Therefore, it is easy to show that:
\begin{align}
    \mathcal{Z}_{\rm{out}}(y,\omega, V) = \frac{1}{2}\erfc\left(-\frac{y\omega}{\sqrt{2(\noisevar+V)}}\right), && Z_{w}(b, A) = \frac{e^{\frac{b^2}{1+A}}}{1+A}
\end{align}
and therefore:
\begin{align}
f_{\text{out}}(y,\omega,V) = \frac{2y~\mathcal{N}(\omega y|0, V+\noisevar)}{\erfc\left(-\frac{y\omega}{\sqrt{2(\noisevar+V)}}\right)}, && f_{w}(b, A) = \frac{b}{1+A}
    \end{align}
This form of the prior allow us to simplify some of the equations considerably:
\begin{align}
    q_{\bo}^{t+1} = \mathbb{E}_{(w_{\star}, \xi)}\left[f_{w}(\hat{q}^{t}w_{\star}+\sqrt{\hat{q}^{t}}\xi, \hat{q}^{t})^2\right] = \frac{1}{1+\hat{q}^{t}_{\bo}}
\end{align}
which is the equation found in Theorem \ref{thm:jointstats}. The other equation cannot be closed analytically, however it can be considerably simplified:
\begin{align}
    \hat{q}_{\bo} &= -\alpha \mathbb{E}_{(z, \omega), \xi}\left[\partial_{\omega}f_{\text{out}}(f_{0}(z + \tau \xi), \omega, V^{t})\right]\\
    &=\frac{2}{\pi}\frac{\alpha}{1+\noisevar-q_{\bo}^{t}}\int_{\mathbb{R}}\dd z~\mathcal{N}\left(z\Big| 0, \frac{q_{\bo}^{t}}{2(1+\noisevar-q^{t}_{\bo})}\right)\frac{e^{-2z^2}}{\erfc(z)\erfc(-z)}
\end{align}

\subsection{ERM estimation}
For ERM, the estimation likelihood $P_{\text{out}}$ and prior $P_{w}$ are related to the loss and penalty functions:
\begin{align}
    P_{\text{out}}(y|x) = e^{-\beta\ell(y, x)}, && P_{w}(w) = e^{-\beta r(w)}.
\end{align}
\noindent where the parameter $\beta > 0$ is introduced for convenience, and should be taken to infinity. Focusing on the regularisation part and redefining $(b, A) \to (\beta b, \beta A)$
\begin{align}
\mathcal{Z}_{w}(b, A) = \int_{\mathbb{R}}\dd w~ e^{-\beta(\frac{A}{2}w^2-b w + r(w))} \underset{\beta\to\infty}{\asymp} e^{\beta\left[\frac{b^2}{2A} - \mathcal{M}_{A^{-1}r}(A^{-1}b)\right]}
\end{align}
\noindent where we have used Laplace's method and defined the \emph{Moreau envelope}:
\begin{align}
\mathcal{M}_{\noisestr f}(x) = \underset{z\in\mathbb{R}}{\rm{min}}\left[\frac{1}{2\noisestr}(x-z)^2 + f(z)\right]	\\
\end{align}
Therefore,
\begin{align}
    f_{w}(b, A) = \lim\limits_{\beta\to\infty}\frac{1}{\beta}\partial_{b}\log{Z_{w}(b,A)} = \prox_{A^{-1}r}(A^{-1}b)
\end{align}
\noindent where we have defined the \emph{proximal operator}:
\begin{align}
    \text{prox}_{\noisestr f}(x) = \underset{z\in\mathbb{R}}{\text{argmin}}\left[\frac{1}{2\noisestr}(x-z)^2 + f(z)\right]
\end{align}
\noindent and used the well-known property $\partial_{x}\mathcal{M}_{\tau f}(x) = -\frac{1}{\tau}\left(\text{prox}_{\tau f}(x) - x\right)$. In particular, for the $\ell_2$-penalty $r(w) = \sfrac{\lambda}{2}w^2$, we have:
\begin{align}
    \prox_{\sfrac{\lambda}{2}(\cdot)^2}(x) = \frac{x}{1+\lambda} && \Leftrightarrow  &&  f_{w}(b, A) = \frac{b}{\lambda+A}
\end{align}
The simple form of the regularization allow us to simplify the state evolution equations considerably:
\begin{align}
\begin{cases}
\tilde{V}^{t+1} &= \mathbb{E}_{(w_{\star}, \xi)}\left[\partial_{b}f_{w}(\hat{\tilde{m}}^{t}w_{\star}+\sqrt{\hat{\tilde{q}}^{t}}\xi, \hat{\tilde{V}}^{t})\right] = \frac{1}{\lambda+\hat{\tilde{V}}}\\
\tilde{q}^{t+1} &= \mathbb{E}_{(w_{\star}, \xi)}\left[f_{w}(\hat{\tilde{m}}^{t}w_{\star}+\sqrt{\hat{\tilde{q}}^{t}}\xi, \hat{\tilde{V}}^{t})^2\right] = \frac{\hat{m}^{2}+\hat{\tilde{q}}}{(\lambda+\hat{\tilde{V}})^2}\\
\tilde{m}^{t+1} &= \mathbb{E}_{(w_{\star}, \xi)}\left[f_{w}(\hat{\tilde{m}}^{t}w_{\star}+\sqrt{\hat{\tilde{q}}^{t}}\xi, \hat{\tilde{V}}^{t})w_{\star i}\right] = \frac{\hat{m}}{\lambda+\hat{\tilde{V}}}
\end{cases}
\end{align}
which are the equations found in Theorem \ref{thm:jointstats}. A similar discussion can be carried for the loss term, and yields in general:
\begin{align}
    f_{\text{out}}(y,\omega, V) = V^{-1}\left(\prox_{\tau\ell(y,\cdot)}(x)-x\right)
\end{align}
Unfortunately, the logistic loss $\ell(y, x) = \log(1+e^{-yx})$ does not admit a closed form solution for the proximal, and therefore for a given $(y,\omega, V)$ we need to compute it numerically. 

\section{Proof of theorems}
\label{sec:app:proofs}
A possible route for proving our result is to give a rigorous proof of the cavity equations. Instead, we shall use a shortcut, and  leverage on recent progresses for both the ERM cavity results \cite{thrampoulidis2018precise,sur_modern_2018,candes_phase_2018,montanari2020generalization,aubin_generalization_2020,loureiro_learning_2021}),  the Bayes performances \cite{barbier2019adaptive,barbier_optimal_2019}, as well as on the performance of GAMP \cite{rangan2011generalized,javanmard2013state,gerbelot2021graphbased}.

\subsection{GAMP optimality}
\label{AMP-proof}
The optimally of GAMP is a direct consequence of the generic results concerning its performance (the state evolution in \cite{rangan2011generalized,javanmard2013state}) and the characterization of the Bayes performance in \cite{barbier_optimal_2019}. G-a works, one considers a sequence of inference problems indexed by the dimension $d$, with data ${\mathcal D}_d$ (which are defined in section \ref{sec:setting} for our purpose). As $d$ increases, both GAMP performances and Bayes errors converge with high probability to the same deterministic limit given by the so-called "replica", or "state evolution" equations. 

To simplify the notation, all our statements involving the asymptotic limit $d \to \infty$ are implicitly defined for such sequences, and the convergence is assumed to be in terms of probability. 


Let us prove that, indeed, GAMP estimates for posterior probability are asymptotically exact with high probability. First, we note that the estimation of the Bayes posterior probability for the signs corresponds to finding the estimators that minimize the MMSE. Indeed consider, for fixed data (this remains  true averaging over data), the mean squared error for an estimator $\hat Y({\bf X})$:
\begin{equation}
    {\rm MSE}(\hat Y({\bf X})) = {\mathbb E}_{Y,{\bf X}} \left[(Y-\hat Y({\bf X}))^2\right] = {\mathbb E}_{\bf X}  {\mathbb E}_{Y|{\bf X}} \left[(Y-\hat Y({\bf X}))^2\right] 
\end{equation}
The mean square error is given by using the posterior mean \cite{cover1991elements}, as can be seen immediately differentiating with respect to $\hat Y$ (for a given {\bf x}), so that: 
\begin{equation}
  \hat Y_{\rm Bayes}({\bf x}) = {\mathbb E}_{Y|X={\bf x}} [Y] = 2 {\mathbb P}_{Y|X={\bf x}}(Y=1)-1
\end{equation}
The Bayes estimator for the posterior probability is thus the MMSE estimator. We see here that the estimation of the posterior mean of $Y$ is  equivalent to the estimation of the probability it takes value one; both quantities are thus trivially related.

We can now use Proposition 2, page $12$ in \cite{barbier_optimal_2019}, that shows that indeed GAMP efficiently achieves Bayes-optimality for the MMSE on $Y$:
\begin{theorem}[GAMP generalisation error, \cite{barbier_optimal_2019}]
Consider a sequence of problems indexed by $d$, with data ${\mathcal D}_d$ in dimension $d$, then we have that GAMP estimator asymptotically achieves the Minimal Mean Square Error in estimating the error on new label $Y$. That is, with high probability:
\begin{equation}
    \lim_{d \to \infty} \mathbb E_{Y,{\bf X}|{\mathcal D}_d} \left[(Y-\hat Y_{\rm GAMP}({\bf X},{\mathcal D}_d)^2 \right] = {\rm MMSE}(Y)
\end{equation}
where $\hat Y_{\rm GAMP}({\bf x},{\mathcal D})=2p-1$, and $p = \hat f^{AMP}({\bf x})$ (eq.~\ref{def:amp}),
with $\hat{\vec{c}}_{\amp}^{\top}(\vec{x} \odot \vec{x})=1-q$, with $q$ a fixed point of \eqref{def:eq-bo}.
\end{theorem}

The fact that GAMP asymptotically achieves the MMSE, coupled with the uniqueness of the Bayes estimator, implies the GAMP estimator for $p$ is arbitrary close to the Bayes estimated for $p$, with high probability over new Gaussian samples, as $d \to \infty$. More precisely, we can use the following lemma:

\begin{lemma}[Bounds on differences of estimators for $Y$]
\label{lemma:bound1}
Consider a sequence of estimation problems indexed by $d$ with data ${\mathcal D}_d$. If a (sequence of) estimators $\hat f_d({\bf x})$ achieves the MMSE performance of $\hat g_d^{\rm Bayes}({\bf x})$  as $d \to \infty$ for Gaussian distributed ${\bf x}$, then 
\begin{equation}
    \lim_{d \to \infty} {\mathbb E}_{\bf X} |f_d({\bf X})-g^{\rm Bayes}_d({\bf X})|^2 \to 0 
\end{equation}
\end{lemma}
\begin{proof}
The Bayes estimator $g^{\rm Bayes}_d({\bf})$ is the minimum of the MMSE, therefore for any other estimator $f_d({\bf X})$ we have
\begin{equation}
    {\mathbb E}\left[(Y-f_d({\bf X})^2 \right]  \ge 
        {\mathbb E} \left[(Y-g^{\rm Bayes}_d({\bf X})^2 \right]\,.
\end{equation}
We have, denoting $\delta_d(X)=f_d({\bf X})-g^{\rm Bayes}_d({\bf X})$
\begin{eqnarray}
  {\mathbb E}\left[(Y-f_d({\bf X})^2 \right] &=& 
  {\mathbb E}\left[(Y-g^{\rm Bayes}_d({\bf X})+\delta_d(X))^2 \right] \\
  &=& {\rm MMSE} +  {\mathbb E}\left[\delta_d(X)^2 + 2\delta_n(X)
  (Y-g^{\rm Bayes}_d(X))\right] \\
  &=& {\rm MMSE} +  {\mathbb E}\left[\delta_d(X)^2\right]
  +  {\mathbb E}_{X,\mathcal D}  {\mathbb E}_{Y|X,\mathcal D}
  \left[2\delta_d(X)  (Y-g^{\rm Bayes}_d(X))\right] \\
    &=& {\rm MMSE} +  {\mathbb E}\left[\delta_d(X)^2\right]
     +  {\mathbb E}_{X,\mathcal D}  
  \left[2\delta_n(X)  {\mathbb E}_{Y|X,\mathcal D}[Y-g^{\rm Bayes}_d(X])\right] \\
  &=& {\rm MMSE} +  {\mathbb E}\left[\delta_d(X)^2\right]
\end{eqnarray}
where we have used $g^{\rm Bayes}_d(X)= {\mathbb E}_{Y|X,\mathcal D}[Y]$. Using the fact that the $f_d$ asymptotically achieve MMSE optimality, we thus obtain:
\begin{equation}
   \lim_{d \to \infty} 
   {\mathbb E}_{Y,X,{\mathcal D}}
   \left[|f_d(X)-g^{\rm Bayes}_d(X)|^2\right]  
   \to 0
\end{equation}
\end{proof}
Applying this lemma to the GAMP estimator leads to Lemma \ref{thm:gamp}: with high probability over new sample ${\bf x}$ and learning data $\mathcal D$, the GAMP estimate is asymptotically equivalent to the Bayes one. 

\subsection{Joint density of estimators}
\label{sec:proof_joint_density}
While a possible strategy to prove the second theorem would be to use state evolution to follow our joint GAMP algorithm (thus monitoring the Bayes {\it and} the ERM performance), we shall instead again leverage on recent  progresses on generic proofs of replica equations, in particular the Bayes one (in \cite{barbier_optimal_2019} and the ERM ones (that were the subject of many works recently  \cite{thrampoulidis2018precise,sur_modern_2018,candes_phase_2018,montanari2020generalization,aubin_generalization_2020,loureiro_learning_2021}). Again, all our statements involving the asymptotic limit $d \to \infty$ are implicitly defined for sequences of problems, and the convergence is assumed to be in terms of probability. We start by the following lemma:
\begin{lemma}[Joint distribution of pre-activation]
For a fixed set of data $\mathcal D$, consider the joint random variables (over {\bf X}) 
$\nu = \vec{X}\cdot \wstar,\lambda_{\erm} = \vec{X}\cdot \hat{\vec w}_{\erm},\lambda_{\amp} = \vec{X} \cdot \hat{\vec{w}}_{\amp}$. Then we have
\begin{equation}
    {\mathbb P}(\nu,\lambda_{\amp},\lambda_{\erm}) ={\mathcal N}\left(0,
  \begin{pmatrix}
\frac {\wstar \cdot \wstar}d & \frac{\wstar \cdot \hat{\vec{w}}_{\amp}}d & \frac{\wstar \cdot \hat{\vec w}_{\erm}}d \\
\frac{\hat{\vec{w}}_{\amp} \cdot \wstar}d &  \frac{\hat{\vec{w}}_{\amp} \cdot \hat{\vec w}_{\amp}}d &  \frac{\hat{\vec{w}}_{\amp} \cdot \hat{\vec w}_{\erm}}d \\
\frac{\hat{\vec w}_{\erm} \cdot \wstar}d &  \frac{\hat{\vec w}_{\erm} \cdot \hat{\vec{w}}_{\amp}}d &  \frac{\hat{\vec w}_{\erm} \cdot \hat{\vec w}_{\erm}}d \\
\end{pmatrix}  
        \right)
\label{def:localfields}
\end{equation}
\label{lemma3}
\end{lemma}
\begin{proof}
This is an immediate consequence of the Gaussianity of the new data {\bf x}, with covariance $\sfrac{{\mathbb I}}{d}$.
\end{proof}

We now would like to know the asymptotic limit of the parameters of this distribution, for large $d$. While we have $\frac {w_{\star} \cdot w_{\star}}d \to \rho$, the other overlap have a deterministic limit given by the replica equations. For empirical risk minimisation, this has been proven in the aforementioned series of works, but we shall here use the notation of \cite{loureiro_learning_2021} and utilize use the following results:

\begin{theorem}[ERM overlaps \cite{thrampoulidis2018precise,aubin_generalization_2020,loureiro_learning_2021}]
Consider a sequence of inference problem indexed by the dimension $d$, then with high probability:
\begin{equation}
   \lim_{d \to \infty} \frac{\hat {\vec{w}}_{\erm} \cdot \wstar}d \to m,\qquad  \lim_{d \to \infty} \frac{\hat {\vec{w}}_{\erm} \cdot \hat {\vec{w}}_{\erm}}d \to q_{\erm}
\end{equation}
With $q_{\erm}$ and $m$ solutions of the self-consistent equations \eqref{def:eq-erm-hats} in the main text.
\end{theorem}
GAMP is tracked by its state evolution \cite{javanmard2013state}, and is known to achieve the Bayes overlap:
\begin{theorem}[Bayes overlaps \cite{barbier_optimal_2019}]
Consider a sequence of inference problem indexed by the dimension $d$, then with high probability:
\begin{align}
   \lim_{d \to \infty} \frac{\hat {\vec{w}}_{\amp} \cdot \wstar}d \to q_{\bo}, \qquad  \lim_{d \to \infty} \frac{\hat {\vec{w}}_{\amp} \cdot \hat {\vec{w}_{\amp}}}d \to q_{\bo}
\end{align}
With $q_{\bo}$ given by the self-consistent Equation \eqref{def:eq-bo}.
\end{theorem}

The only overlap left to control is thus $Q = \sfrac{\hat{\vec{w}}_{\amp} \cdot \hat{\vec{w}}_{\erm}}{d}$. We shall noz prove that it is also concentrating, with high probability, to $m$. To do this, we first prove the following lemma for the overlap between the Bayes estimate $\vec{w}_{\bo} = \mathbb E_{W|{\mathcal D}}[{\vec W}]$ and any other vector ${\bf V}$, possibly dependent on the data:
\begin{lemma}[Nishimori relation for Bayes overlaps]
\begin{equation}
\mathbb E_{{\mathcal D}} \left[\vec{w}_{\bo} \cdot {\bf V}(\mathcal D)\right] = 
\mathbb E_{{\mathcal D},W^*} \left[{\bf w}^* \cdot {\bf V}(\mathcal D)\right]
\end{equation}
\end{lemma}
\begin{proof}
The proof is an application of Bayes formula, and an example of what is often called a Nishimori equality in statistical physics:
\begin{eqnarray}
\mathbb E_{{\mathcal D},W^*} \left[{\bf w}^* \cdot {\bf V}(\mathcal D)\right] &=&
\mathbb E_{{\mathcal D}} E_{W^*|{\mathcal D}} \left[{\bf w}^* \cdot {\bf V}(\mathcal D)\right] \\
&=&  \mathbb E_{{\mathcal D}}  \left[(E_{W^*|{\mathcal D}} {\bf w}^*) \cdot {\bf V}(\mathcal D)\right] = \mathbb E_{{\mathcal D}} \left[\vec{w}_{\bo} \cdot {\bf V}(\mathcal D)\right]  
\end{eqnarray}
\end{proof}

From this lemma, we see immediately that, in expectation
\begin{equation}
    \lim_{d \to \infty} {\mathbb E}\left[\frac{\vec{w}_{\erm} \cdot \wstar}{d}\right] = \lim_{d \to \infty} {\mathbb E}\left[\frac{\vec{w}_{\erm} \cdot \vec{w}_{\bo}}{d}\right] = m
\end{equation}
Additionally, we  already know that the left hand side concentrates.  It is easy to see that the right hand side does as well:
\begin{lemma}[Concentration of the overlap $Q$]
\begin{equation}
  \lim_{d \to \infty} {\mathbb E}\left[\left(\frac{\vec{w}_{\bo} \cdot \vec{w}_{\erm}}{d}\right)^2\right] =  \lim_{d \to \infty} {\mathbb E}\left[\frac{\vec{w}_{\bo} \cdot \vec{w}_{\erm}}{d} \right]^2
\end{equation}
\end{lemma}
\begin{proof}
The proof again uses Nishimori identity. 
\begin{eqnarray}
{\mathbb E}\left[\left(\frac{\vec{w}_{\bo} \cdot \vec{w}_{\erm}}{d}\right)^2\right] &=& {\mathbb E}\left[\left(\frac{\vec{w}_{\bo} \cdot \vec{w}_{\erm}}{d}\right)\left(\frac{\vec{w}_{\bo} \cdot \vec{w}_{\erm}}{d}\right)\right] \\ &=& 
{\mathbb E}_{\mathcal D}\left[\left(\frac{ {\mathbb E}_{W|\mathcal D} W  \cdot \vec{w}_{\erm}}{d}\right)\left(\frac{ {\mathbb E}_{W|\mathcal D} W\cdot \vec{w}_{\erm}}{d}\right)\right] \\
&=& 
{\mathbb E}_{\mathcal D}
{\mathbb E}_{W_1,W_2|\mathcal D}
\left[\left(\frac{W_1 \cdot \vec{w}_{\erm}}{d}\right)\left(\frac{ W_2\cdot \vec{w}_{\erm}}{d}\right)\right] \\
&=& 
{\mathbb E}_{\mathcal D,{\vec w}^*}
\left[\left(\frac{{\vec w}^* \cdot \vec{w}_{\erm}}{d}\right)\left(\frac{ {\mathbb E}_{W|\mathcal D} W\cdot \vec{w}_{\erm}}{d}\right)\right] \\
&=& 
{\mathbb E}_{\mathcal D,{\vec w}^*}
\left[\left(\frac{{\vec w}^* \cdot \vec{w}_{\erm}}{d}\right)\left(\frac{ {\vec w}_{\bo} \cdot \vec{w}_{\erm}}{d}\right)\right] 
\end{eqnarray}
Then, from Cauchy-Schwartz we have
\begin{eqnarray}
{\mathbb E}\left[\left(\frac{\vec{w}_{\bo} \cdot \vec{w}_{\erm}}{d}\right)^2\right]^2 &\le& {\mathbb E}\left[\left(\frac{\vec{w}_{\bo} \cdot \vec{w}_{\erm}}{d}\right)^2\right]
{\mathbb E}\left[\left(\frac{\vec{w}^* \cdot \vec{w}_{\erm}}{d}\right)^2\right] \\
{\mathbb E}\left[\left(\frac{\vec{w}_{\bo} \cdot \vec{w}_{\erm}}{d}\right)^2\right] &\le&
{\mathbb E}\left[\left(\frac{\vec{w}^* \cdot \vec{w}_{\erm}}{d}\right)^2\right]
\end{eqnarray}
and as $d \to \infty$, we can use the concentration of the right hand side to $m$ to obtain
\begin{eqnarray}
\lim_{d \to \infty} {\mathbb E}\left[\left(\frac{\vec{w}_{\bo} \cdot \vec{w}_{\erm}}{d}\right)^2\right] &\le&
m^2
\end{eqnarray}
so that, given the second moment has to be larger or equal to its (squared) mean: 
\begin{eqnarray}
\lim_{d \to \infty} {\mathbb E}\left[\left(\frac{\vec{w}_{\bo} \cdot \vec{w}_{\erm}}{d}\right)^2\right] &=& m^2
\end{eqnarray}
\end{proof}
We have thus proven that the overlap $Q$ concentrates in quadratic mean to $m$ as $d\to \infty$: with high probability, it is thus asymptotically equal to $m$. We shall now prove that $\vec{w}_{\bo}$ can be approximated by $\vec{w}_{\amp}$. In fact, given the concentration of overlap, it will be enough to prove that:
\begin{equation}
    \lim_{d \to \infty} {\mathbb E}_{{\mathcal D}_d}  \frac{\hat {\vec w}_{\amp}({\mathcal D}_d)\cdot {\vec w}_{\erm}({\mathcal D})}{d} = 
    \lim_{d \to \infty} {\mathbb E}_{{\mathcal D}_d}  \frac{\hat {\vec w}_{\bo}({\mathcal D}_d)\cdot {\vec w}_{\erm}({\mathcal D})}{d}
\end{equation}

This can be done in two steps. First, similarly as in section \ref{AMP-proof}, we use the fact that GAMP achieves Bayes optimality for the estimation of $W^*$ \cite{barbier_optimal_2019}. This leads to the following lemma

\begin{lemma}[Bounds on differences of estimators for ${\bf w}$]
\begin{equation}
    \lim_{d \to \infty} {\mathbb E}_{\mathcal D} \frac{\|\vec{w}_{\amp}-\vec{w}_{\bo}\|^2}d \to 0 
\end{equation}
\end{lemma}
\begin{proof}
The proof proceeds similarly as in lemma \ref{lemma:bound1}. Denoting $\delta {\vec w}({\mathcal D})=\vec{w}_{\amp}({\mathcal D})-\vec{w}_{\bo}({\mathcal D})$ we write
\begin{align}
    &{\mathbb E}_{{\mathcal D},{\vec W^*}} \frac{\|{\vec w}_{\amp}({\mathcal D})-{\vec w}^*\|_2^2}d =
     {\mathbb E}_{{\mathcal D},{\vec W^*}} \frac{\|{\vec w}_{\bo}({\mathcal D})+\delta {\vec w}({\mathcal D})-{\vec w}^*\|_2^2}d\\
     &=  {\mathbb E}_{{\mathcal D},{\vec W^*}} \frac{\|{\vec w}_{\bo}({\mathcal D})-{\vec w}^*\|_2^2}d +  {\mathbb E}_{{\mathcal D}} \frac{\|\delta {\vec w}({\mathcal D})\|_2^2}d + \frac 1d 2 {\mathbb E}_{{\mathcal D}} {\mathbb E}_{{\vec w}^*|{\mathcal D}} \left[\delta w({\mathcal D})({\vec w}^*-{\vec w}_{\bo}\right ] \\
     &=  {\mathbb E}_{{\mathcal D}} \frac{\|\delta {\vec w}({\mathcal D})\|_2^2}d
\end{align}
Using the optimality of GAMP for the MMSE yields the lemma.
\end{proof}
We can now prove the equality of overlaps
\begin{lemma}
\begin{equation}
   \lim_{d \to \infty} {\mathbb E}_{{\mathcal D}_d}  \frac{\hat {\vec w}_{\amp}({\mathcal D}_d)\cdot {\bf V}({\mathcal D})}{d} = 
    \lim_{d \to \infty} {\mathbb E}_{{\mathcal D}_d}  \frac{\hat {\vec w}_{\bo}({\mathcal D}_d)\cdot {\bf V}({\mathcal D})}{d}
    \end{equation}
\end{lemma}
\begin{proof}
The proof is an application of Cauchy-Schwartz inequality:
\begin{eqnarray}
\left| {\mathbb E}_{{\mathcal D}_d} \left[ \frac{(\hat {\vec w}_{\amp}-\hat {\vec w}_{\bo})({\mathcal D}_d)\cdot {\bf V}({\mathcal D})}{d} \right]\right| 
&\le& 
\sqrt{{\mathbb E} \frac{\|{\bf V}\|_2^2}d  
{\mathbb E} \frac{\|{\vec w}_{\bo}-{\vec w}_{\amp}\|_2^2}d  
}
\end{eqnarray}
taking the limit $d \to \infty$ yields the lemma.
\end{proof}
Applying the lemma to the ERM estimator, and using the concentration of overlaps, finally leads to
\begin{lemma}[Asymptotic Joint distribution of pre-activation]
Asymptotically, and with high probability over data, the joint distribution of pre-activation is asymptotically given by
\begin{equation}
\lim_{d\to \infty}    {\mathbb P}(\nu,\lambda_{\amp},\lambda_{\erm}) ={\mathcal N}\left(0,\begin{pmatrix}
\rho & q_{\bo} & m \\
q_{\bo} & q_{\bo} & m \\
m & m & q_{\erm} \\
\end{pmatrix}\right)
\end{equation}
\end{lemma}
To obtain Theorem \ref{thm:jointstats}, one simply applies the change of variable 
\begin{align}
(\nu,\lambda_{\amp},\lambda_{\erm}) &\to (f_{\star}(\nu), \hat{f}_{\amp}(\lambda_{\amp}, \hat{f}_{\erm}(\lambda_{\erm}) \\
&= \left(\sigma_{\star}(\sfrac{\nu}{\tau}), \sigma_{\star}(\sfrac{\lambda_{\amp}}{\tau'}), \sigma(\lambda_{\erm})) \right)
\end{align}

\subsection{Proof of Theorem \ref{thm:calibration}}
\label{sec:proof_thm_calibration}

\paragraph{Proof of Equation \eqref{eq:thm_cal_erm}} Consider the local fields $(\nu, \lambda_{\erm}, \lambda_{\amp})$ defined in Equation \eqref{def:localfields}. As shown above, these local fields follow a Gaussian distribution with covariance matrix $\Sigma$ given in Equation \eqref{eq:def_sigma}. Then, $(\nu, \lambda_{\erm})$ follows a bivariate Gaussian and the density of $\nu$ conditioned on $\hat{f}_{\erm}(\vec{x}) = \sigma(\lambda_{\erm}) = p$ follows the Gaussian distribution with mean $\mu = \frac{m}{q_{\erm}}\sigma^{-1}(p)$ and variance $v^2 = \rho - \frac{m^2}{q_{\erm}}$. Then, 
\begin{align}
    \mathbb{E}_{x} \left[ f_{\star}(\vec{x}) | \hat{f}_{\erm}(\vec{x}) = p \right] &= \int d\nu \frac{1}{2} \erfc\left(- \frac{\nu}{\sqrt{2\noisevar}}\right) \mathcal{N}(\nu | \mu, v^2) \\
    &= \frac{1}{2} \erfc\left(- \frac{\mu}{\sqrt{2 \left( \noisevar + v^2 \right)}}\right) = \frac{1}{2} \erfc\left(-\frac{\frac{m}{q_{\erm}}\sigma^{-1}(p)}{\sqrt{2(1 - \frac{m^2}{q_{\erm}} + \noisevar)}}\right)\\
    &= \sigma_{\star} \left( \frac{\frac{m}{q_{\erm}}\sigma^{-1}(p)}{\sqrt{1 - \frac{m^2}{q_{\erm}} + \noisevar}} \right)
\end{align}
which yields Equation \eqref{eq:thm_cal_erm}.
We used the property that, for any $a, b$, 
\begin{equation}
    \int \erf(ax + b) \mathcal{N}(x | \mu, \sigma^2) \dd{x} = \erf\left(\frac{a\mu + b}{\sqrt{1 + 2 a^2 \sigma^2}}\right)
\end{equation}

\paragraph{Proof of Equation \eqref{eq:thm_cal_bo}}
We use the same computation as in the previous paragraph: since the conditioned on the Bayes local field $\hat{f}_{\bo}(\vec{x}) = \sigma_{\star}(\frac{\lambda_{\amp}}{\sqrt{\noisevar + 1 - q_{\bo}}}) = p$, the teacher local field is Gaussian with mean $\mu = \sqrt{\noisevar + 1 - q_{\bo}}\sigma_{\star}^{-1}(p)$ and variance $v^2 = 1 - q_{\bo}$. As before, we have \begin{align}
    \mathbb{E}_{\vec{x}} \left[ f_{\star}(\vec{x}) | \hat{f}_{\bo}(\vec{x}) = p \right] &= \sigma_{\star}\left( \frac{\mu}{\sqrt{\noisevar + v^2}} \right) \\
    &= \sigma_{\star}\left(\frac{\sqrt{\noisevar + 1 - q_{\bo}} \sigma_{\star}^{-1}(p) }{\sqrt{\noisevar + 1 - q_{\bo}}} \right) = p 
\end{align}
Hence the result of Equation \eqref{eq:thm_cal_bo}.

\paragraph{Proof of Equation \eqref{eq:thm_cal_erm_bo}}
The proof follows the same structure as the previous paragraphs: conditioned on $\sigma(\lambda_{\erm}) = p$, the law of $\lambda_{\amp}$ is $\mathcal{N}(\frac{m}{q_{\erm}}\sigma^{-1}(p), q_{\bo} - \frac{m^2}{q_{\erm}})$ and 
\begin{align}
    \mathbb{E}_{\vec{x}}\left[ \hat{f}_{\bo}(\vec{x}) | \hat{f}_{\erm}(\vec{x}) = p \right] &= \int \sigma_{\star}\left( \frac{- x }{\sqrt{\noisevar + 1 - q}}\right) \mathcal{N}(x | \frac{m}{q_{\erm}}\sigma^{-1}(p), q_{\bo} - \frac{m^2}{q_{\erm}}) \\
    &= \sigma_{\star} \left( \frac{ \frac{m}{q_{\erm}}\sigma^{-1}(p)  }{\sqrt{\noisevar + 1 - q_{\bo} + (q_{\bo} - \frac{m}{q_{\erm}})}} \right) \\
    &= \sigma_{\star}\left( \frac{\frac{m}{q_{\erm}}\sigma^{-1}(p)}{\sqrt{1 - \frac{m^2}{q_{\erm}} + \noisevar}} \right) = \mathbb{E}_{x} \left[ f_{\star}(\vec{x}) | \hat{f}_{\erm}(\vec{x}) = p \right]
\end{align}


\newpage

\section{Additional figures}
\label{sec:additional_figures}




\subsection{Logistic regression uncertainty supplement}

Figure \ref{fig:multiple_densities_bo_erm} complements Figure \ref{fig:erm_joint_density_lambda_0} from the main text by showing the same plot as the right panel in Figure~\ref{fig:erm_joint_density_lambda_0}  for other values of sample complexity $\alpha$ and noise $\tau$. We observe that at zero regularization the logistic regression is overconfident in all the depicted cases, in particular so at small $\alpha$ and small noise.  

\begin{figure}[ht!]
    \centering
    \includegraphics[width=0.8\textwidth]{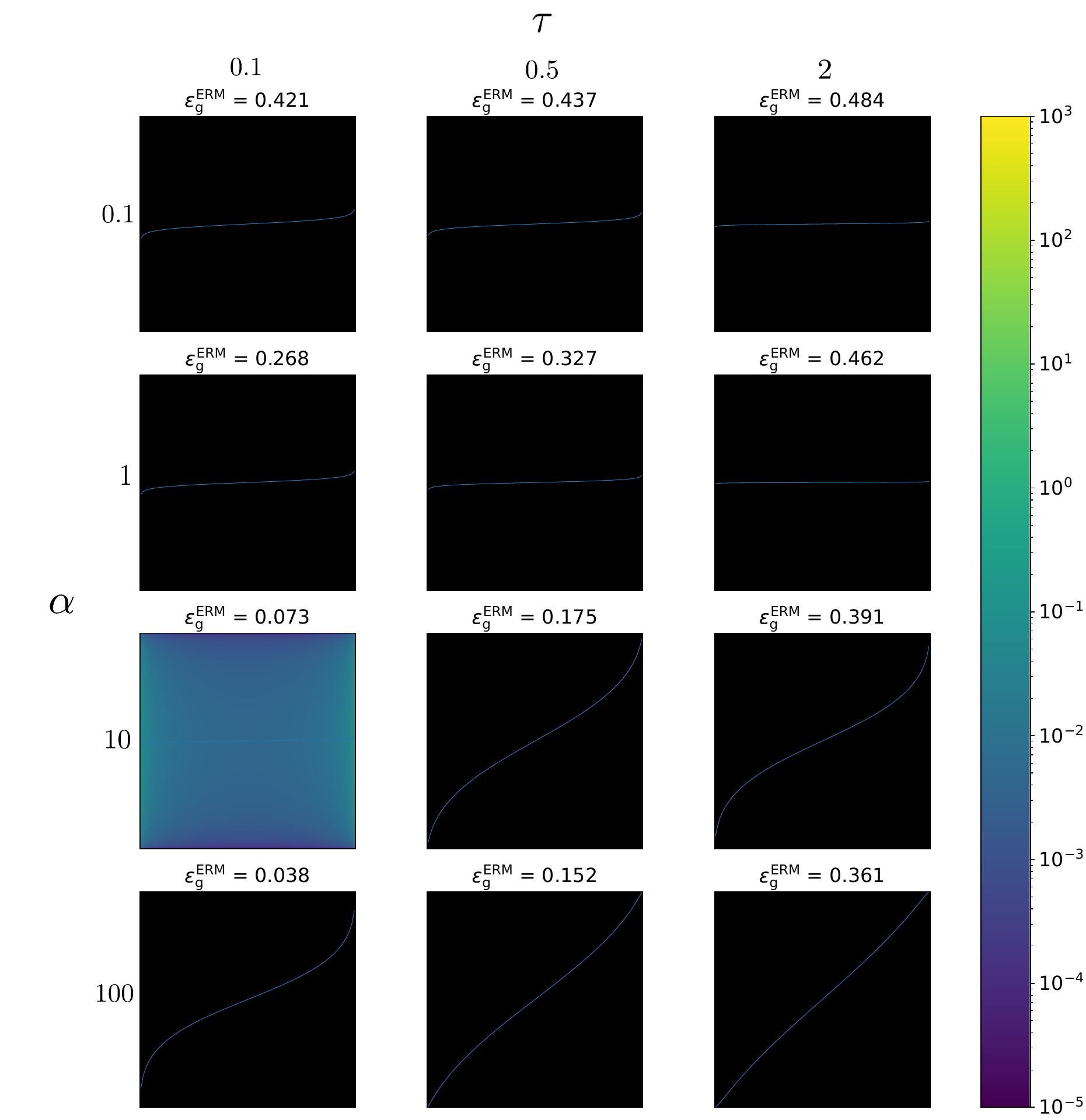}
    \caption{Joint density of $\hat{f}_{\erm}$ (x-axis) and $\hat{f}_{\bo}$ (y-axis), at $\lambda = 0^{+}$. Blue curve is the mean of $\hat{f}_{\bo}$ at fixed $\hat{f}_{\erm}$.
    The test error of ERM is indicated above the corresponding plot. The test errors of Bayes for the same parameters are indicated in Figure \ref{fig:multiple_densities_main}.}
    \label{fig:multiple_densities_bo_erm}
\end{figure}

\subsection{Choosing optimal regularization supplement}
Here we give additional illustration related to the section \ref{sec:choosing_lambda_optimally} in the Main text.

In figure \ref{fig:calibration_lambda}, the calibration $\Delta_p$ is shown as a function of $\lambda$ at different levels $p$ and different noise $\sigma$. First observe that as $\lambda$ grows the logistic regression is going from overconfident $\Delta_p>0$ to underconfident $\Delta_p<0$. For $\lambda \rightarrow \infty$, we have $\Delta_p \rightarrow p - 1$.
Further, we observe that the value of $\lambda$ at which the calibration is zero (the best calibration) has only mild dependence on the value of $p$.  
Finally, the vertical lines mark the values of regularization that minimize the validation error $\lambdaerror$, and loss $\lambdaloss$. We see that $\lambdaloss$ is closer to the well-calibrated region, and that at small $\alpha$ this difference in more pronounced.

\begin{figure}[h!]
    \centering
    \includegraphics[width = 0.5\columnwidth]{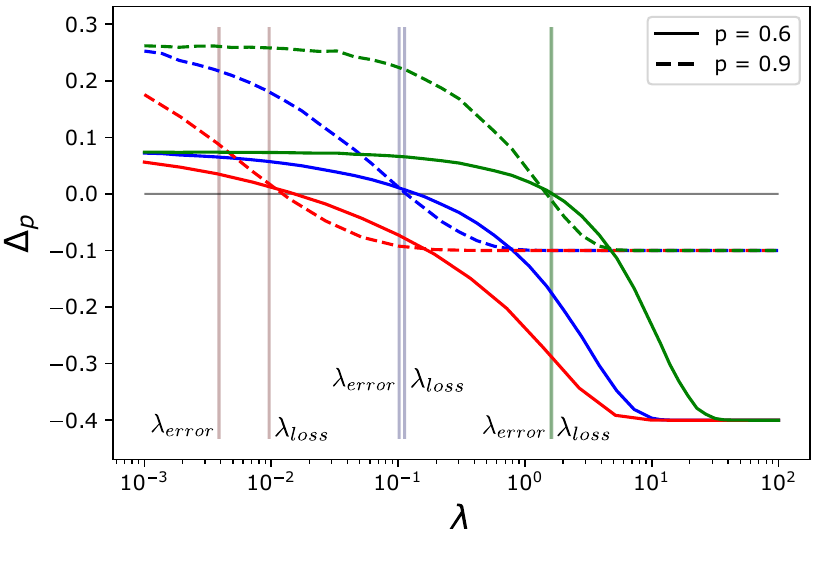}
    \caption{Calibration $\Delta_p$ for $p = 0.9$ and $p = 0.6$ as a function of $\lambda$, for $\noisestr = 0$ (red curve) , $\noisestr = 0.5$ (blue curve), and $\noisestr = 2$ (green curve), at $\alpha = 5$. Vertical lines correspond to $\lambdaerror$ and $\lambdaloss$ defined in \ref{sec:choosing_lambda_optimally}. For $\noisestr = 2$, $\lambdaerror$ and $\lambdaloss$ differ by only $10^{-2}$ and look indistinguishable on the plot.}
    \label{fig:calibration_lambda}
\end{figure}

The left panel of Figure \ref{fig:optimal_lambda} compares $\lambdaerror$ and $\lambdaloss$ when $\noisestr = 0.5$. 
The right panel of Figure \ref{fig:optimal_lambda}
then shows that the test error at $\lambdaloss$ and $\lambdaerror$ are extremely close, with the difference being plot in the insert. 

\begin{figure}[ht!]
    \centering
    \subfigure[]{\includegraphics[scale=0.5]{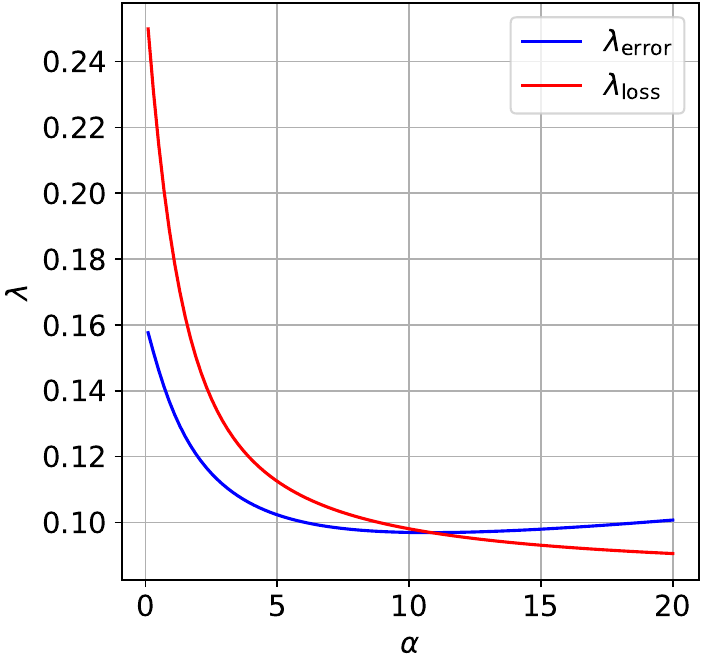}}
    \subfigure[]{\includegraphics[scale=0.5]{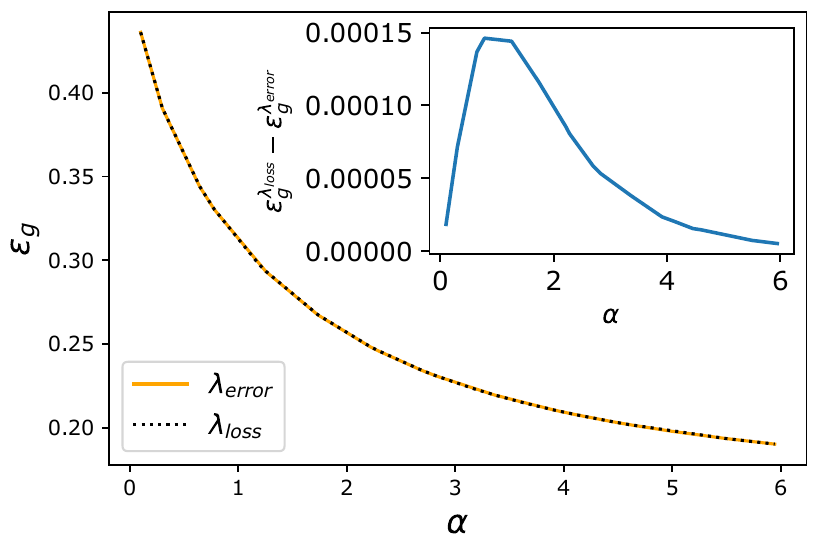}}
    \caption{Left: optimal penalization for logistic regression as a function of the sample complexity $\alpha$, for $\noisestr = 0.5$. 
    Right: Test error at optimal $\lambda$ for $\sigma = 0.5$, as a function of $\alpha$. Orange line (respectively black dotted line) corresponds to $\lambda$-error (respectively $\lambda$-loss). The two curves are indistinguishable on the plot. The blue curve in the inset shows $\varepsilon_g^{\lambdaloss} - \varepsilon_g^{\lambdaerror}$ as a function of $\alpha$: it appears that the difference is around $\sim 10^{-4}$.
    }
    \label{fig:optimal_lambda}
\end{figure}


Figure \ref{fig:multiple_densities_lambda} depicts the joint density of $\hat{f}_{\erm}$ (x-axis) and $\hat{f}_{\bo}$ (y-axis) for several values of the regularization $\lambda$ and the noise~$\tau$. As $\lambda$ increases, we observe that the logistic regression changes from overconfident to underconfident, 
    as we could also observe in figure \ref{fig:calibration_lambda}.

\begin{figure}[ht!]
    \centering
    \includegraphics[width=0.8\textwidth]{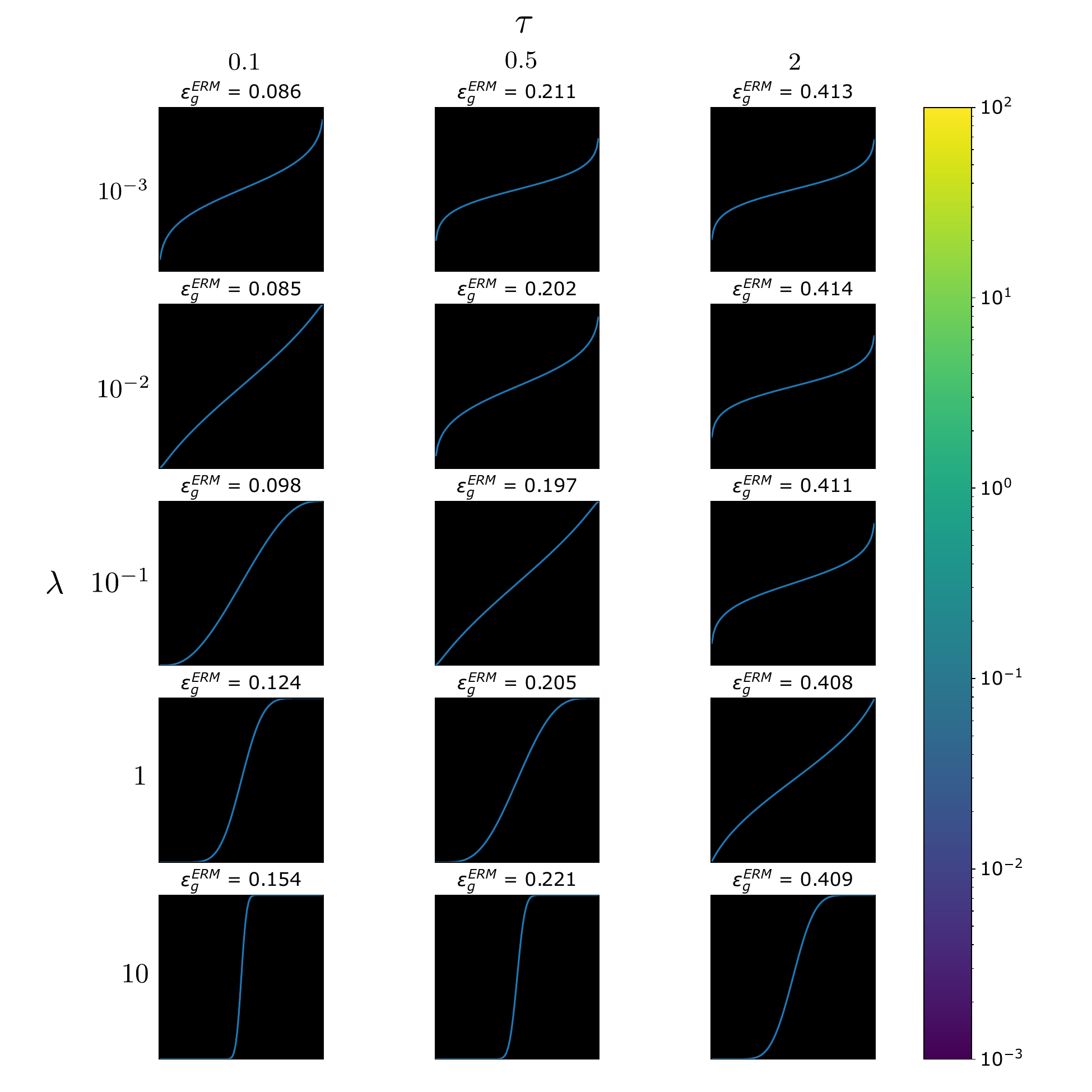}
    \caption{Joint density of $\hat{f}_{\erm}$ (x-axis) and $\hat{f}_{\bo}$ (y-axis) at $\alpha = 5$. The best possible test errors are respectively $\varepsilon^{\star}_g = 0, 0.148, 0.352$ for $\noisestr = 0, 0.5, 2$. For the Bayes estimator with $\alpha = 5$, the test errors are $\varepsilon^{\bo}_g = 0.083, 0.198, 0.402$}
    \label{fig:multiple_densities_lambda}
\end{figure}

Next in Figure \ref{fig:multiple_densities_bo_erm_error_loss}  we depict the densities for $\lambdaerror$ and $\lambdaloss$ for different values of $\alpha$ and $\tau$. We observe an overall improvement in the calibration for these optimal regularizations. 

\begin{figure}[h!]
    \centering
    \includegraphics[width=0.49\textwidth]{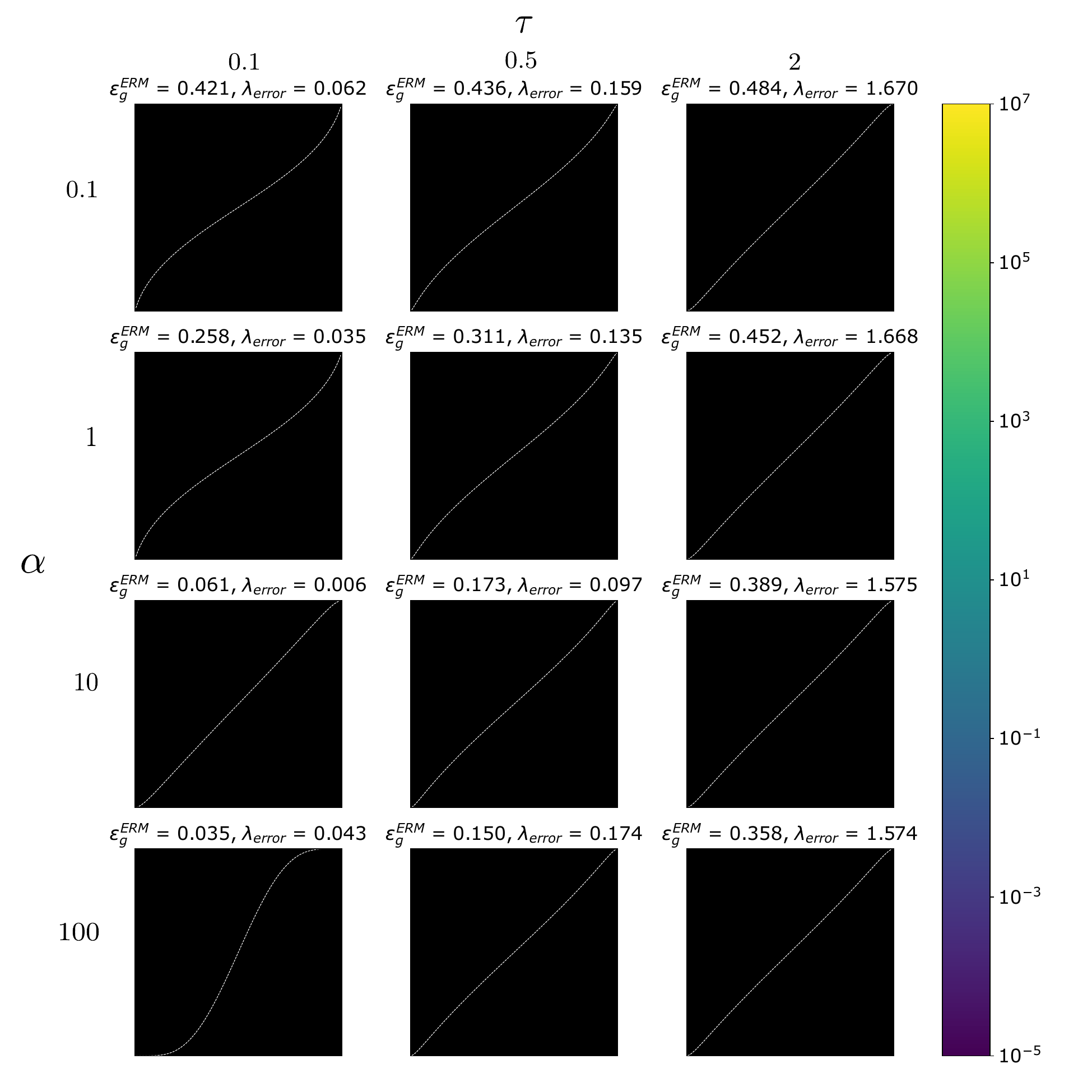}
    \includegraphics[width=0.49\textwidth]{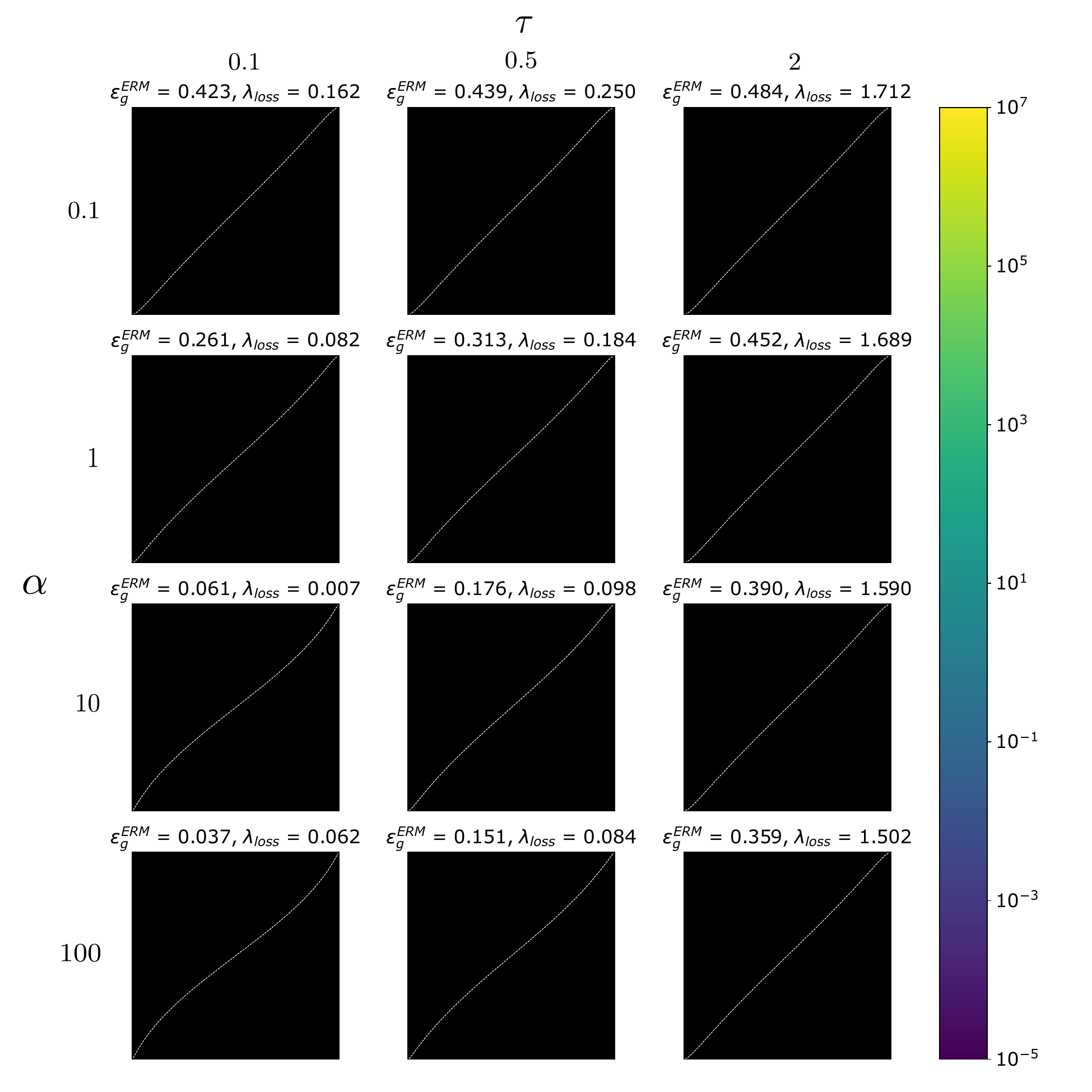}
    \caption{Left: Joint density $\rho_{\erm, \bo}$, at $\lambda = \lambdaerror$.
    $\lambdaerror$ and the test error of ERM are indicated above the corresponding plot. Right:Joint density $\rho_{\erm, \bo}$, at $\lambda = \lambdaloss$.
    $\lambdaloss$ and the test error of ERM are indicated above the corresponding plot. }
    \label{fig:multiple_densities_bo_erm_error_loss}
\end{figure}




\newpage 
\section{Comparison to the data generated by logit model}
\label{app:logit}

As mentioned before, our state evolution equations can be adapted to data generated by the logit model, as studied in \cite{bai_dont_2021}. The likelihood is defined in Equation~\eqref{eq:logistic_data_model}. Since this change only concerns the data distribution, Algorithm~\ref{alg:gamp} is unchanged. However, state evolution is changed in the update of $\hat{m}, \hat{q}, \hat{V}$: the partition function $\mathcal{Z}_0$ is now 
\begin{equation}
    \mathcal{Z}_0(y, \omega, V) = \int \dd{z} ~\sigma(y \times z) \mathcal{N}(z | \omega, V)
\end{equation}

Note also that the expression of the calibration is now 
\begin{equation}
    \Delta_p = p - \int \dd{x} \sigma(x) \mathcal{N}(x | \sfrac{m}{q} \times \sigma^{-1}(p), \rho - \sfrac{m^2}{q})
    \label{eq:calibration_logistic_data}
\end{equation}

\subsection{Behaviour at $\lambda = 0^{+}$}

In \cite{bai_dont_2021}, it has been shown that as the sampling ratio $\alpha$ goes to $\infty$, the unpenalized logistic classifier is calibrated when the data is generated by the \textit{logit} model
\begin{equation}
    \mathbb{P}(y_{\star} = 1) = \sigma(\mathbf{w}_{\star} \cdot \mathbf{x}) \label{eq:logistic_data_model}
\end{equation}

In this section, we numerically recover the results from \cite{bai_dont_2021} i.e  the unpenalized logistic estimator is calibrated asymptotically and the calibration decreases as $\sfrac{1}{\alpha}$. Figure~\ref{fig:logistic_lambda_0_calibration} plots the calibration at $p = 0.75, 0.9$ and $0.99$ for $\alpha \in [10, 10^4]$. One can observe a decay of $\Delta_p$ with a power law, which confirms that with logistic data, the unpenalized logistic classifier is asymptotically calibrated at all levels. Fitting a linear model on these curves gives slopes equal to $-0.99, -1.00, -1.04$ for $p = 0.75, 0.9, 0.99$ respectively, which numerically validates the $\sfrac{1}{\alpha}$ rate derived in \cite{bai_dont_2021}.

\begin{figure}[!ht]
    \centering
    \includegraphics[width = 0.5 \textwidth]{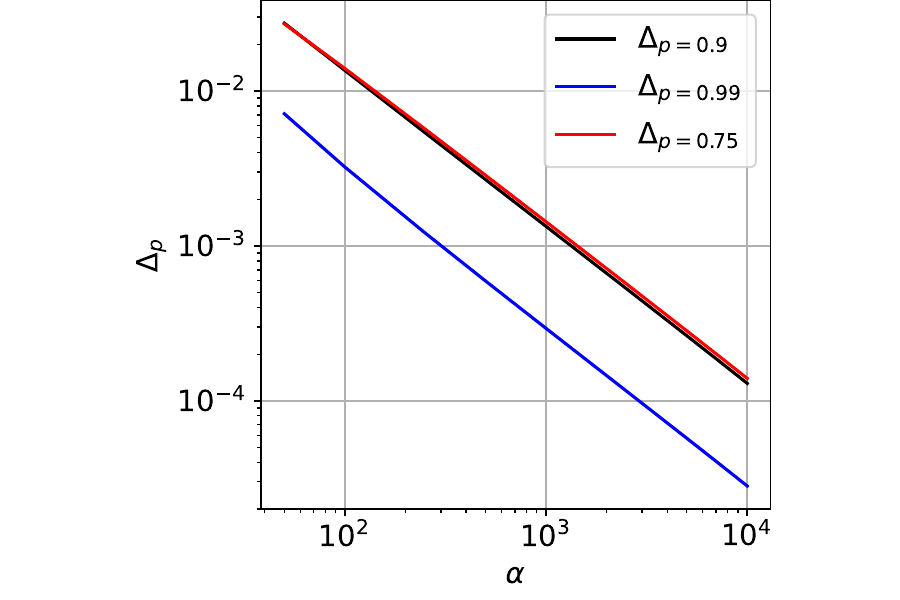}
    \caption{Calibration $\Delta_p$ at $p = 0.75, 0.9$ and $0.99$ of logistic regression with the logit model as a function of $\alpha$. The plots are given in log-log scale. On this scale, the curves have respective slopes $-0.99, -1.01, -1.04$}
    \label{fig:logistic_lambda_0_calibration}
\end{figure}

We compare here to the calibration with probit data, at $\noisestr = 0.5$. In particular, we exhibit that the logistic classifier cannot be calibrated at all levels $p$. Indeed, as $\alpha \to \infty$, it can be noted that $\cos(\hat{\mathbf{w}}_{\erm}, \wstar) =  \sfrac{m^2}{q} \to_{\infty} 1$. Moreover, we observe that $\sfrac{m}{q} = \sfrac{m^2}{q} \times \sfrac{1}{m} \to_{\infty} m_{\infty} \coloneqq \lim m $.
Using the expression for calibration from Theorem \ref{thm:calibration}, we get that for $p > \sfrac{1}{2}$,
\begin{equation}
    \Delta_p \to_{\infty} p - \sigma_{\star}\left( \frac{\sigma^{-1}(p)}{\tau \times m_{\infty}} \right)
\end{equation}
And deduce that 
\begin{equation}
    \Delta_p = 0 \Leftrightarrow \frac{\sigma_{\star}^{-1}(p)}{\sigma^{-1}(p)} = \frac{1}{\tau \times m_{\infty}}
\end{equation}
Noting $r(p) \coloneqq \frac{\sigma_{\star}^{-1}(p)}{\sigma^{-1}(p)}$, we get the condition 
\begin{equation}
    p = r^{-1}(\frac{1}{\tau \times m_{\infty}})
\end{equation}
With $\noisestr = 0.5$, we numerically get $m_{\infty} \simeq 3.53 \Rightarrow \noisestr \times m_{\infty} \simeq 1.76$
The level $p_{0}$, defined as the only $p  > \sfrac{1}{2}$ such that $\Delta_{p} = 0$, is thus
\begin{equation}
    p_0 = r^{-1}(\frac{1}{\tau \times m_{\infty}}) \simeq r^{-1}(0.57) \simeq 0.937 
\end{equation}
For $\sfrac{1}{2} < p < p_0$ (respectively $1 > p > p_0$), $\Delta_p > 0$ (respectively $\Delta_p < 0$). This can be observed in Figure \ref{fig:large_alpha_calibration} where we have plotted $\Delta_p$ for several levels. For $p \neq p_0$, the calibration seems to converge a finite value. On the other hand, at $p = p_0$, $\Delta_p$ converges to 0 as a power-law.

\begin{figure}[h!]
    \centering
    \subfigure{\includegraphics[width=0.42 \textwidth]{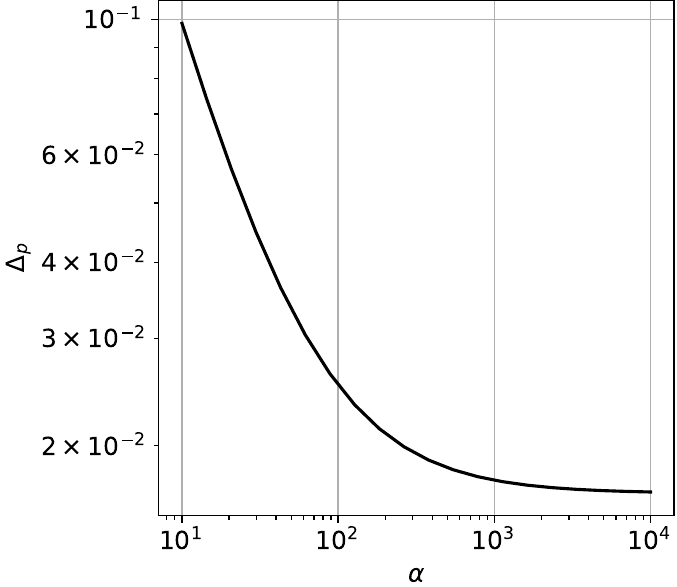}}
    \subfigure{\includegraphics[width=0.40 \textwidth]{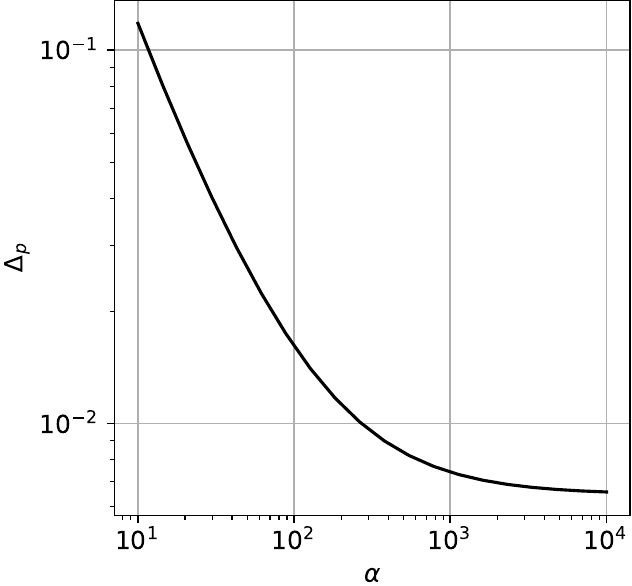}}
    
    \subfigure{\includegraphics[width=0.42 \textwidth]{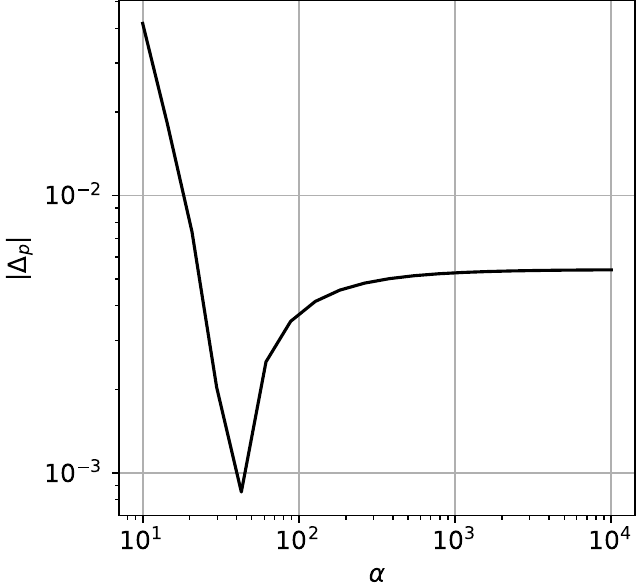}}
    \subfigure{\includegraphics[width=0.40 \textwidth]{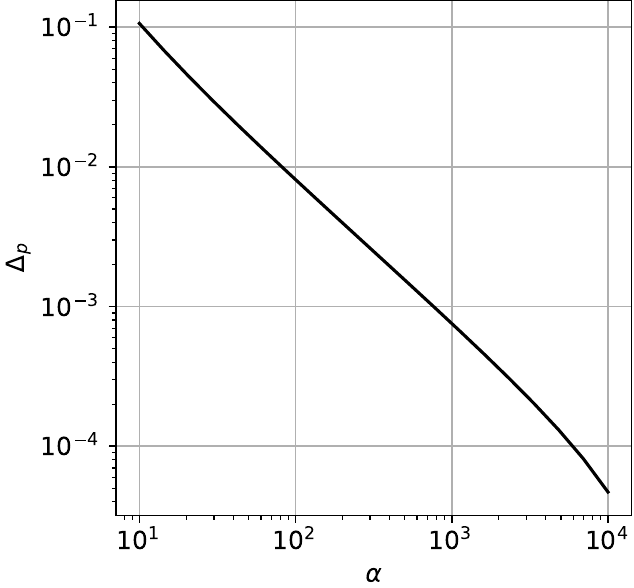}}
    \caption{Calibration for logistic regression with $\lambda = 0^{+}, \noisestr = 0.5$, at four different levels: $p = 0.75$ (Top left), $p = 0.9$ (Top right), $p = 0.99$ (Bottom left) and $p = p_0 \simeq 0.937$ (Bottom right). For $p = 0.99$, $\Delta_p$ becomes negative around $\alpha \simeq 50$ so the absolute value of $\Delta_p$ has been plotted instead. C seems to converge to zero for $p = p_0$ only.}
    \label{fig:large_alpha_calibration}
\end{figure}

\subsection{Behaviour a $\lambda = 1$, $\lambda_{\rm error}$ and $\lambda_{\rm loss}$}

In this section, we adapt the theoretical results of Figure~\ref{fig:calibration_optimal_lambda} to the logit data model: we compute $\lambda_{\rm error}$ and $\lambda_{\rm loss}$ and plot their respective test errors and calibration. Note the definition of the test error and loss in this setting: 
\begin{align}
    \begin{cases}
        \varepsilon_{g} &= \sum_y \mathbb{E}_{\xi \sim \mathcal{N}(0, 1)} \left[ \mathcal{Z}_0(y, \sfrac{m}{\sqrt{q}}\xi, 1 - \sfrac{m^2}{q}) \delta(\sign(\xi) = y) \right] \\
        \mathcal{L}_{g} &= - \sum_y \mathbb{E}_{\xi \sim \mathcal{N}(0, 1)} \left[ \mathcal{Z}_0(y, \sfrac{m}{\sqrt{q}}\xi, 1 - \sfrac{m^2}{q}) \log \sigma(y \times \sqrt{q}\xi) \right]
    \end{cases}
\end{align}


Moreover, with the logit data model, the empirical risk at $\lambda = 1$, now has a Bayesian interpretation. The risk corresponds to the logarithm of the posterior distribution on $\vec{w}$, up to a normalization constant, because $\vec{w}_{\star}$ is sampled from a Gaussian with identity covariance. At $\lambda = 1$, the empirical risk minimizer $\hat{\vec{w}}_{\rm erm}$ is the Maximum A Posteriori (MAP). In this section, we compare the performance of logistic regression with the two different optimal regularizations and with $\lambda = 1$.

The left panel of Figure~\ref{fig:optimal_lambda_logistic} shows the value of $\lambda_{\rm error}$ and $\lambda_{\rm loss}$. As with the probit model, $\lambda_{\rm loss} > \lambda_{\rm error}$. Note also that both optimal values are bigger than $1$ for this range of $\alpha$. The right panel shows their respective test error $\varepsilon_{g}$. As with the probit model, $\lambda_{\rm error}$ has a lower error than $\lambda_{\rm loss}$. Not surprisingly, $\lambda = 1$ has worse test error than both optimal $\lambda$.
Left panel of Figure~\ref{fig:optimal_lambda_logistic_calibration} shows the calibration with the three different regularizations at $p = 0.75$. We observe that $\lambda = 1$ yields an overconfident estimator (consistent with the fact that $\lambda_{\rm error}$ and $\lambda_{\rm loss}$ are both bigger than $1$), and as  before, $\lambdaloss$ is less confident than $\lambdaerror$. 
Remark that an underconfident estimator is not necessarily better than an overconfident one, and the calibration $\Delta_p$ is only a measure on one level $p$. To compare the different estimators more fairly, we can thus use a metric called \textit{Expected Calibration Error} defined as 
\begin{equation}
    {\rm ECE} \coloneqq \mathbb{E}_{\hat{f}(\mathbf{x})}\left( | \Delta_{\hat{f}(\mathbf{x})} | \right) = \int \dd{p} |\Delta_p| \frac{\mathcal{N}(\sigma^{-1}(p) | 0, q_{\erm})}{p(1 - p)}
\end{equation}
The ECE measures the average of $|\Delta_p|$ at all levels $p$ weighted by the probability that $\hat{f}(\mathbf{x}) = p$. In other words, at a given level $p$, if $\mathbb{P}(\hat{f}(\mathbf{x}) = p) = 0$, the ECE of the estimator will not be affected by the calibration of the estimator at $p$.
The right panel of Figure \ref{fig:optimal_lambda_logistic_calibration} plots the ECE as a function of $\alpha$ for $\lambda = 1$, $\lambdaerror$ and $\lambdaloss$. We again observe that $\lambdaloss$ has a lower ECE than $\lambdaerror$, which confirms that optimizing $\lambda$ for the test loss yields a more calibrated estimator. Moreover, $\lambda = 1$ yields an estimator with the worst ECE, which is coherent with the left panel: at $p = 0.75$, the absolute value of its calibration is higher than $\lambda_{\rm error}$ and $\lambda_{\rm loss}$. 
Our numerical results show that even if we know the prior distribution on the posterior and the likelihood, using only a point estimate for the parameter (here the maximum a posteriori) yields an overconfident estimator.

\begin{figure}[h!]
    \centering
    \subfigure{\includegraphics[width = 0.32 \textwidth]{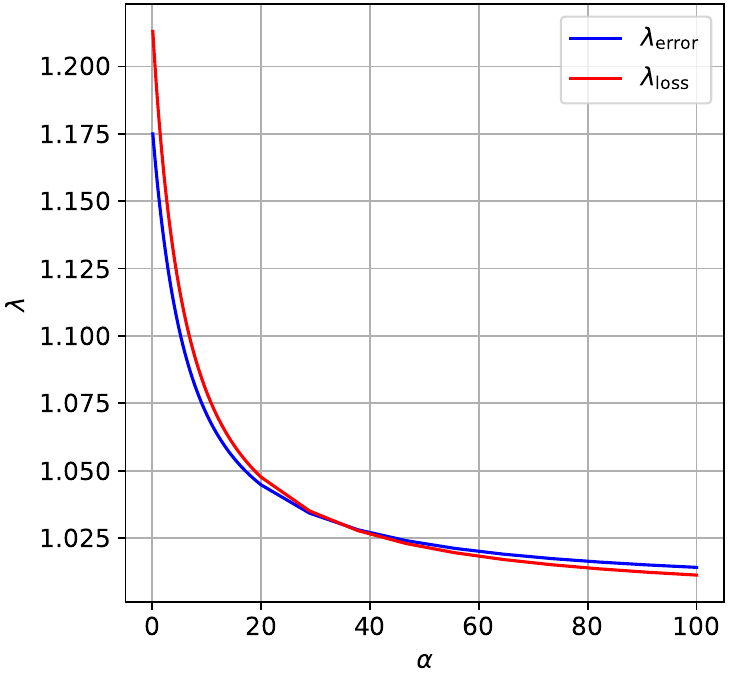}}
    \subfigure{\includegraphics[width = 0.32 \textwidth]{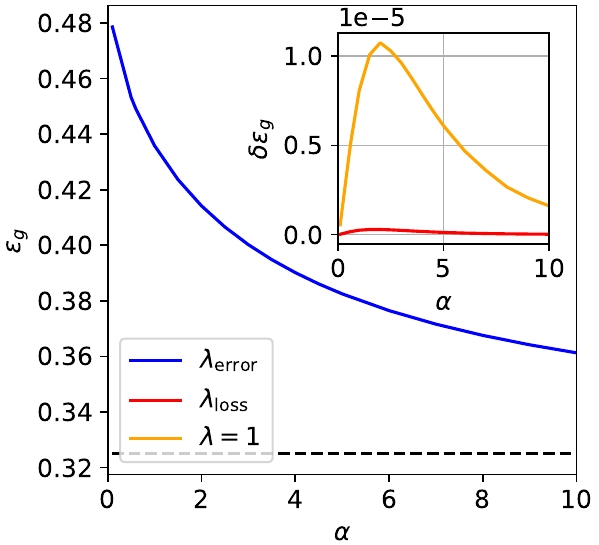}}
    \caption{Left: Values of $\lambda_{\rm error}$ and $\lambda_{\rm loss}$ as a function of $\alpha$ for the logistic data model. Center: Values of the test error $\varepsilon_g$ for $\lambda_{\rm error}$ (blue curve) . The inset plots the difference of test error $\delta \varepsilon_{g, {\rm loss}} \coloneqq \varepsilon_g(\lambda_{\rm loss}) - \varepsilon_g(\lambda_{\rm error})$ (red curve) and $\delta \varepsilon_{g, 1} \coloneqq \varepsilon_g(\lambda = 1) - \varepsilon_g(\lambda_{\rm error})$ (orange curve). Right: Calibration at $p = 0.75$ of logistic regression on logistic data, for $\lambda = 1$, $\lambdaerror$ and $\lambdaloss$. The curves are given by running state evolution. }
    \label{fig:optimal_lambda_logistic}
\end{figure}

\begin{figure}[h!]
    \centering
   \subfigure{\includegraphics[width = 0.4 \textwidth]{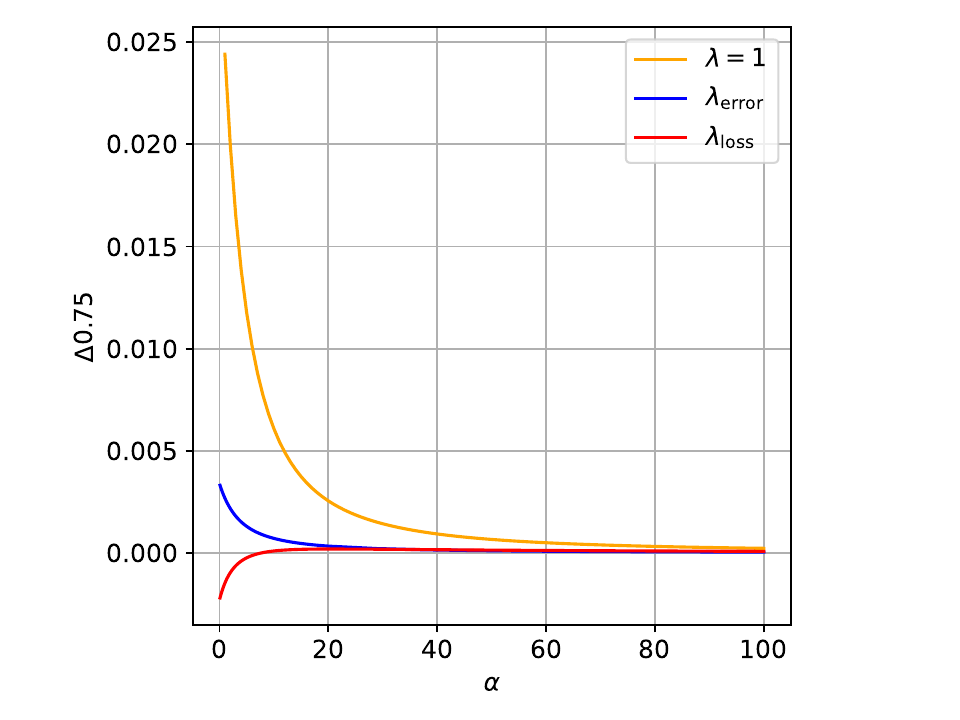}}
\subfigure{\includegraphics[width = 0.4 \textwidth]{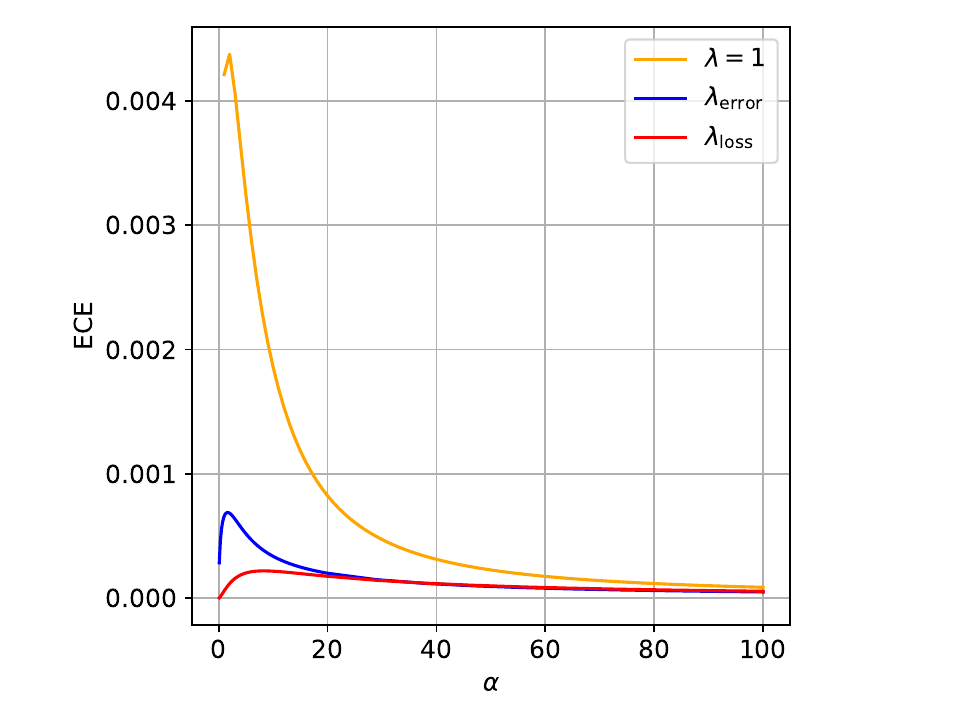}}
    \caption{Left: Calibration at $p = 0.75$ of logistic regression on logistic data, for $\lambda = 1$, $\lambdaerror$ and $\lambdaloss$. The curves are given by running state evolution. Right: Expected Calibration Error (ECE) for $\lambda = 1$, $\lambdaerror$, $\lambdaloss$. The lower ECE, the better.}
    \label{fig:optimal_lambda_logistic_calibration}
\end{figure}

\end{document}